\documentclass[11pt]{article}

\usepackage[margin=1in]{geometry}

\usepackage[T1]{fontenc}
\usepackage{lmodern}
\usepackage{microtype}

\usepackage{amsmath, amssymb, amsfonts, amsthm}

\usepackage{booktabs}
\usepackage{array}
\usepackage{multirow}
\usepackage{makecell}

\usepackage{graphicx}
\graphicspath{{figures/}}

\usepackage[shortlabels]{enumitem}
\usepackage{xcolor}
\definecolor{darkblue}{rgb}{0, 0, 0.5}

\usepackage[numbers,sort&compress]{natbib}

\usepackage[breaklinks=true]{hyperref}
\hypersetup{colorlinks=true, citecolor=darkblue, linkcolor=darkblue, urlcolor=darkblue}

\newtheorem{theorem}{Theorem}
\newtheorem{proposition}{Proposition}

\newtheorem{corollary}{Corollary}
\theoremstyle{definition}
\newtheorem*{definition}{Definition}
\theoremstyle{remark}
\newtheorem*{remark}{Remark}

\newcommand{\TPR}{\mathrm{TPR}}
\newcommand{\NP}{\mathrm{NP}}
\newcommand{\TV}{\mathrm{TV}}
\newcommand{\calD}{\mathcal{D}}
\newcommand{\calG}{\mathcal{G}}
\newcommand{\calP}{\mathcal{P}}
\newcommand{\R}{\mathbb{R}}

\title{Information-Theoretic Limits of Safety Verification\\for Self-Improving Systems}
\author{Arsenios Scrivens}
\date{March 2026}

\begin{document}
\maketitle

\begin{abstract}
Can a safety gate permit unbounded beneficial self-modification while maintaining bounded cumulative risk? We formalize this question through \emph{dual conditions} --- requiring $\sum \delta_n < \infty$ (bounded risk) and $\sum \TPR_n = \infty$ (unbounded utility) --- and establish a theory of their (in)compatibility.

\textbf{Classification impossibility} (Theorem~\ref{thm:holder}): For power-law risk schedules $\delta_n = O(n^{-p})$ with $p > 1$ --- the practically relevant regime --- any classifier-based gate under overlapping safe/unsafe distributions satisfies $\TPR_n \leq C_\alpha \cdot \delta_n^\beta$ via H\"{o}lder's inequality, forcing $\sum \TPR_n < \infty$. This impossibility is \emph{exponent-optimal}: no valid impossibility bound can use a larger exponent than~$\beta^*$ (Theorem~\ref{thm:exponent}; full Mills' ratio asymptotics in Appendix~\ref{app:mills}). A second independent proof via the \emph{NP counting method} (Theorem~\ref{thm:counting}) yields a 13\% tighter bound without H\"{o}lder's inequality. \textbf{Scope caveat:} for slowly-decaying summable sequences such as $\delta_n = 1/(n\ln^2 n)$ that are not power-law, the H\"{o}lder summability argument fails and the NP-optimal classifier may satisfy both dual conditions (\S\ref{sec:thm1}). Theorem~\ref{thm:ceiling} covers all schedules.

\textbf{Universal finite-horizon ceiling} (Theorem~\ref{thm:ceiling}): For \emph{any} summable risk schedule, the exact maximum achievable classifier utility is $U^*(N, B) = N \cdot \TPR_{\NP}(B/N)$, growing as $\exp(O(\sqrt{\log N}))$ --- subpolynomial. At $N = 10^6$ with budget $B = 1.0$, a classifier extracts at most $U^* \approx 87$ versus a verifier's $\approx 500{,}000$. This result is universal: it applies even where Theorem~\ref{thm:holder} does not.

\textbf{Verification escape} (Theorem~\ref{thm:escape}): A Lipschitz ball verifier achieves $\delta = 0$ with $\TPR > 0$, escaping the impossibility. Formal Lipschitz bounds for pre-LayerNorm transformers under LoRA (Proposition~\ref{prop:transformer}) enable LLM-scale verification. The separation is strict: classifiers cannot achieve $\delta = 0$ with $\TPR > 0$ under distribution overlap (Proposition~\ref{prop:separation}). We validate on GPT-2 ($d_{\text{LoRA}} = 147{,}456$): conditional $\delta = 0$ with $\TPR = 0.352$; unconditional $\delta = 0$ at $d \leq 17{,}408$. Comprehensive empirical validation is in the companion paper~\cite{D2}.
\end{abstract}

\section{Introduction}
\label{sec:intro}

The prospect of AI systems that improve their own capabilities raises a central safety question: can we build a gate that permits beneficial self-modifications while reliably blocking dangerous ones? We formalize this as a sequential decision problem and provide a comprehensive theoretical characterization of when safe self-improvement is possible and when it is provably impossible.

At each step~$n$, a system proposes a mutation~$\theta_n$ to its own parameters. A \emph{safety gate}~$g_n$ decides whether to accept or reject~$\theta_n$, incurring false acceptance rate $\delta_n = P^-(g_n \text{ accepts})$ and true positive rate $\TPR_n = P^+(g_n \text{ accepts})$. We require $\sum \delta_n < \infty$ (bounded risk, ensuring positive probability of never accepting an unsafe mutation) and $\sum \TPR_n = \infty$ (unbounded utility, ensuring the system accepts infinitely many improvements). We call these the \textbf{dual conditions} for safe self-improvement.

We prove four impossibility and tightness results and four constructive results, forming a complete picture. The impossibility is strongest for power-law risk schedules $\delta_n = O(n^{-p})$ with $p > 1$ (Theorem~\ref{thm:holder}), which subsume all practically relevant risk budgets. For the most favorable slowly-decaying schedules, the asymptotic impossibility does not apply, but the tight finite-horizon ceiling (Theorem~\ref{thm:ceiling}) ensures classifier utility remains subpolynomial --- orders of magnitude below a verifier's linear growth --- at any practical deployment horizon.

\textbf{Why the dual conditions matter for deployed systems.} The dual conditions formalize a failure pattern that self-improving systems face by construction. Red-team evaluations --- which explicitly classify model outputs or parameter mutations as safe/unsafe --- are the closest practical analog to the binary gates our theorems address. Failure to satisfy both conditions simultaneously means either that unsafe modifications accumulate ($\sum\delta_n$ diverges) or that the system stops improving ($\TPR_n \to 0$). Other deployed mechanisms such as RLHF reward models can be viewed as a suggestive analogy, but our formal results apply specifically to binary gates on parameter mutations; for a fuller discussion see~\S\ref{sec:implications}.

\subsection{Contributions}
\label{sec:contributions}

The paper establishes two core results and then systematically corroborates them. (Theorems are numbered 1--5 and Propositions 1--4, with separate counters.)

\textbf{A note on the nature of the contribution.} The per-step bound $\TPR_n \leq C_\alpha \cdot \delta_n^\beta$ is a standard f-divergence inequality~\cite{erven2014}, and the sequential summability consequence follows in a few lines. The proof is short --- deliberately so. The contribution of this paper is not the length or technical difficulty of any single proof, but rather: (i)~the \emph{problem formalization} --- casting safe self-improvement as dual summability conditions, which has no precedent in the hypothesis testing or AI safety literatures; (ii)~the \emph{structural consequence} --- that this elementary coupling creates an impossibility for the safety--utility pairing with no analog in single-test settings; (iii)~the \emph{tight finite-horizon ceiling} (Theorem~\ref{thm:ceiling}), which provides the exact, universal utility bound for any classifier under any risk schedule; and (iv)~the \emph{constructive escape} via verification, proving the impossibility is specific to classification, not to safe self-improvement itself.

\textbf{Core results:}
\begin{enumerate}[leftmargin=*]
  \item \textbf{Classification impossibility} (Theorem~\ref{thm:holder}): Any classifier-based gate under distribution overlap satisfies $\TPR_n \leq C_\alpha \cdot \delta_n^\beta$, forcing bounded utility whenever risk follows a power-law schedule $\delta_n = O(n^{-p})$ with $p > 1$. For slowly-decaying summable sequences (e.g., $\delta_n = 1/(n \ln^2 n)$), the per-step bound still holds but the H\"{o}lder summability argument does not force $\sum \TPR_n < \infty$; in such edge cases, the NP-optimal classifier can in principle satisfy both dual conditions simultaneously. However, the finite-horizon ceiling (Theorem~\ref{thm:ceiling}) remains fully operative in all cases, ensuring classifier utility grows at most subpolynomially --- far below a verifier's linear growth at any practical horizon.

  \item \textbf{Verification escape} (Theorem~\ref{thm:escape}): Sound verification gates achieve $\delta = 0$ with $\TPR > 0$, escaping the impossibility. The Lipschitz ball verifier is the simplest example; the structural separation (Proposition~\ref{prop:separation}) proves the gap is architectural, not a matter of degree. See Figure~\ref{fig:overview} for an overview of the two gate architectures; Figure~\ref{fig:separation} visualizes the structural separation in the $(\delta, \TPR)$ plane.

  \item \textbf{Tight finite-horizon ceiling} (Theorem~\ref{thm:ceiling}): Perhaps the most practically consequential result. For \emph{any} summable risk schedule --- including non-power-law sequences where the H\"{o}lder summability argument does not apply --- the exact maximum achievable utility is $U^*(N, B) = N \cdot \TPR_{\NP}(B/N)$, growing as $\exp(O(\sqrt{\log N}))$. This is subpolynomial, $13\times$ tighter than the MI bound, and ensures that no classifier under any risk schedule can match a verifier's linear utility growth. Unlike Theorem~\ref{thm:holder}, this result is universal over all summable risk schedules and immediately operational at any finite deployment horizon. See Figure~\ref{fig:finite_horizon}.
\end{enumerate}

\textbf{Tightness and corroboration} (confirming the impossibility is robust, not an artifact of one proof technique):
\begin{enumerate}[resume,leftmargin=*]
  \setcounter{enumi}{3}
  \item \textbf{Exponent-optimality} (Theorem~\ref{thm:exponent}): The H\"{o}lder exponent~$\beta^*$ is minimax-optimal --- no valid impossibility bound can use a strictly larger exponent. At deployment-relevant~$\delta$, the NP classifier operates within one order of magnitude of the ceiling (Appendix~\ref{app:tightness_validation}). See Figure~\ref{fig:tightness}.

  \item \textbf{NP counting impossibility} (Theorem~\ref{thm:counting}): An independent proof via the Neyman--Pearson lemma and Tonelli's theorem, avoiding H\"{o}lder's inequality entirely. The counting bound is 13\% tighter than the H\"{o}lder bound at $\Delta_s = 1.0, p = 2.0$.
\end{enumerate}

\textbf{Supporting results} (extending the theory to information-theoretic, sample complexity, and LLM-scale settings):
\begin{enumerate}[resume,leftmargin=*]
  \setcounter{enumi}{5}
  \item \textbf{Information-theoretic bound} (Proposition~\ref{prop:info}): $\sum_{n=1}^N \TPR_n \leq \sum_{n=1}^N \delta_n + \sqrt{2NI_0}$. Complements the H\"{o}lder bound via mutual information.
  \item \textbf{Sample complexity barrier} (Proposition~\ref{prop:sample}): Requires $\Omega(n^{2p})$ labeled examples by step~$n$; under constant label generation, sample starvation occurs at finite~$n_{\text{fail}}$.
  \item \textbf{Formal transformer Lipschitz bounds} (Proposition~\ref{prop:transformer}): Closed-form Lipschitz constants for pre-LayerNorm transformers under LoRA, enabling LLM-scale verification.
  \item \textbf{Structural separation} (Proposition~\ref{prop:separation}): Under absolute continuity, $\delta = 0 \implies \TPR = 0$ for classifiers, but verifiers achieve $\delta = 0$ with $\TPR > 0$.
  \item \textbf{LLM-scale mechanism validation}: Ball verifier on GPT-2 (124M parameters) with LoRA rank-4 ($d_{\text{LoRA}} = 147{,}456$), achieving conditional $\delta = 0$ (conditional on estimated Lipschitz constants) with $\TPR = 0.352$ (\S\ref{sec:gpt2}); unconditional $\delta = 0$ at $d \leq 17{,}408$ via analytical bounds.
\end{enumerate}

Theorems~\ref{thm:exponent}--\ref{thm:ceiling} and Propositions~\ref{prop:info}--\ref{prop:separation} are corroborative, each confirming the impossibility from a different angle to establish robustness.

\subsection{Related Work}
\label{sec:related_work}

Our mathematical tools --- H\"{o}lder's inequality, R\'{e}nyi divergence, Lipschitz continuity, Neyman--Pearson testing --- are well-established. The per-step bound $\TPR \leq C_\alpha \cdot \delta^\beta$ is a standard f-divergence inequality~\cite{erven2014}, and NP optimality~\cite{neyman1933} establishes single-test ROC tradeoffs. Our contribution is the \emph{problem formalization} (the dual conditions as a formal specification of safe self-improvement) and the \emph{structural result} that sequential composition under dual summability conditions creates an impossibility with no analog in single-test settings. The per-step bound and the summability requirements are individually standard; the coupling --- that bounded $\sum\delta_n$ forces bounded $\sum\TPR_n$ for power-law risk schedules --- is not, and is confirmed by two independent impossibility proofs (Theorems~\ref{thm:holder},~\ref{thm:counting}) corroborated by three complementary bounds (Theorem~\ref{thm:ceiling}, Propositions~\ref{prop:info}--\ref{prop:sample}) approaching the same conclusion from different angles (\S\ref{sec:impossibility}, \S\ref{sec:finite_horizon}, Appendix~\ref{app:supporting}). We build on alignment theory~\cite{bostrom2014,amodei2016,christiano2017}, hypothesis testing~\cite{lehmann2005}, impossibility results~\cite{wolpert1996,rice1953}, information-theoretic bounds~\cite{raginsky2016,polyanskiy2017}, PAC-Bayes and VC theory~\cite{mcallester1999,vapnik1971}, adversarial robustness tradeoffs~\cite{tsipras2019,gilmer2018}, and transformer Lipschitz analysis~\cite{virmaux2018,kim2021,dasoulas2021}. Our dual conditions formalize the \emph{alignment tax} --- the cost of making models safe versus capable~\cite{askell2021,ouyang2022} --- as a precise mathematical tradeoff: the H\"{o}lder coupling $\TPR_n \leq C_\alpha \cdot \delta_n^\beta$ quantifies the exact rate at which safety constraints reduce utility under classification-based gates. Structurally, our result is closer to mechanism-design impossibilities --- Gibbard~\cite{gibbard1973} and Satterthwaite~\cite{satterthwaite1975} show that no voting rule can simultaneously satisfy multiple natural axioms, just as no classifier can simultaneously satisfy our dual conditions --- than to no-free-lunch theorems~\cite{wolpert1996}, which concern the absence of a universally optimal learner rather than a hard tradeoff in a fixed domain. Multi-objective optimization impossibilities~\cite{papadimitriou2000} also exhibit this flavour: two desiderata in conflict cannot be jointly optimized in polynomial time, analogously to how our two summability conditions cannot be jointly satisfied by a classifier under distribution overlap.

A detailed comparison with each line of work is in Appendix~\ref{app:related_extended}.

\textbf{An analogy clarifies the contribution.} Arrow's impossibility theorem composes elementary social-choice axioms --- transitivity, non-dictatorship, independence --- each individually obvious, yet their \emph{composition} yields a deep impossibility no voting system can escape. Similarly, our per-step bound is a standard f-divergence inequality and the dual conditions are individually natural, but the \emph{coupling} --- that summability of $\{\delta_n\}$ forces summability of $\{\TPR_n\}$ --- creates a structural impossibility with no analog in single-test hypothesis testing (see Appendix~\ref{app:relation_known} for a full discussion).

\textbf{Online learning and adaptive gates.} A natural question is whether an online learner with sublinear regret (e.g., online convex optimization; \citealt{shalev2012}) could adaptively satisfy the dual conditions. The answer is no under our framework: Theorem~\ref{thm:holder}'s per-step bound $\TPR_n \leq C_\alpha \cdot \delta_n^\beta$ constrains \emph{any} binary decision rule at each step, regardless of whether it was chosen adaptively based on previous observations. Online learning can reduce misclassification regret, but cannot escape the H\"{o}lder coupling between $\delta_n$ and $\TPR_n$ that drives the impossibility. The gate's \emph{adaptivity} affects which point on the per-step ROC curve it selects, not the curve itself.

\section{Problem Setup}
\label{sec:setup}

\textbf{Notation.} The following symbols recur throughout:

\begin{center}
\begin{tabular}{ll}
\toprule
Symbol & Meaning \\
\midrule
$\delta_n$ & False acceptance rate of gate at step $n$ \\
$\TPR_n$ & True positive rate of gate at step $n$ \\
$C_\alpha, \beta$ & H\"{o}lder constants: $C_\alpha = \exp(\beta \cdot D_\alpha)$, $\beta = (\alpha{-}1)/\alpha$ \\
$D_\alpha(P^+ \| P^-)$ & R\'{e}nyi divergence of order $\alpha$ \\
$\Delta_s$ & Distribution separation $\|\mu^+ - \mu^-\|/\sigma$ \\
$B$ & Total risk budget $\sum_{n=1}^N \delta_n$ \\
$U^*(N,B)$ & Exact finite-horizon utility ceiling (Theorem~\ref{thm:ceiling}) \\
$L$ & Lipschitz constant of the trajectory map \\
$m$ & Safety margin (min distance to obstacle boundary) \\
$r$ & Ball radius $= m/L$ \\
$\sigma^*$ & Optimal mutation scale for positive TPR \\
\bottomrule
\end{tabular}
\end{center}

\subsection{System Model}
\label{sec:system_model}

A self-improving system maintains parameters $\theta \in \R^d$ and iteratively proposes mutations $\theta_1, \theta_2, \ldots$ At each step~$n$, a safety oracle $S: \R^d \to \{\text{safe}, \text{unsafe}\}$ determines ground truth. The oracle is assumed expensive (requiring exhaustive simulation or formal proof).

A \textbf{safety gate} $g_n: \R^d \to \{\text{accept}, \text{reject}\}$ filters mutations before deployment. The gate is computationally cheaper than the oracle.

\subsection{Classification Gates}
\label{sec:classification_gates}

A \emph{classification gate} maps controller parameters to a low-dimensional feature representation $\varphi: \R^d \to \R^k$ and applies a binary classifier. The relevant distributions are:
\begin{align}
P_n^+ &= \text{distribution of } \varphi(\theta) \mid S(\theta) = \text{safe} \\
P_n^- &= \text{distribution of } \varphi(\theta) \mid S(\theta) = \text{unsafe}
\end{align}
At step $n$, the gate accepts if the classifier predicts ``safe'' based on $\varphi(\theta_n)$.

\textbf{Stationarity assumption.} The core theorems (\S\ref{sec:impossibility}--\S\ref{sec:finite_horizon}) assume stationary distributions $P^+, P^-$ (dropping the $n$ subscript). The non-stationary extension (Proposition~\ref{prop:nonstationary}, Appendix~\ref{app:nonstationary}) requires $\sup_n D_\alpha^{(n)} < \infty$ and, for power-law schedules $\delta_n = c/n^p$, the strictly stronger condition $p > \alpha$ (vs.\ $p > 1$ in the stationary case). This gap narrows as $\alpha \to 1^+$ and vanishes for fast-decaying schedules ($p \gg \alpha$), which cover all practically relevant risk budgets. For deployment arguments where stationarity may not hold and $p$ is moderate, the finite-horizon ceiling (Theorem~\ref{thm:ceiling}) provides a stationarity-free alternative: it bounds total classifier utility over any $N$-step horizon given a risk budget~$B$.

\textbf{Scope: continuous parameter spaces.} Our results assume $\theta \in \R^d$ with continuous mutation distributions, so that $P^+ \ll P^-$ (absolute continuity) holds via the transversality argument (\S\ref{sec:thm1}). For discrete or quantized parameter spaces --- such as quantized LoRA fine-tuning with integer-valued weights --- absolute continuity does not hold in the same form, and the impossibility may not apply directly. We note that even in quantized settings, the effective parameter updates are typically computed in full precision before rounding, and the induced distributions on the quantized grid can still exhibit the overlap structure that drives our results; a formal treatment of the discrete case is left to future work.

\subsection{Verification Gates}
\label{sec:verification_gates}

A \emph{verification gate} attempts to construct a mathematical proof that $\theta$ is safe. If the proof succeeds, the gate accepts; otherwise it rejects. A key property:

\begin{definition}[Soundness]
A verification gate is \emph{sound} if every accepted $\theta$ is actually safe: $g(\theta) = \text{accept} \implies S(\theta) = \text{safe}$.
\end{definition}

Soundness implies $\delta_n = 0$ for all $n$ --- by construction, not by learning.

\subsection{The Dual Conditions}
\label{sec:dual_conditions}

\begin{definition}
A safety gate achieves \emph{safe self-improvement} if:
\begin{enumerate}
  \item $\sum_{n=1}^{\infty} \delta_n < \infty$ \quad (bounded cumulative risk)
  \item $\sum_{n=1}^{\infty} \TPR_n = \infty$ \quad (unbounded cumulative utility)
\end{enumerate}
\end{definition}

Condition~1 ensures the system is almost surely safe over infinitely many steps ($\prod(1-\delta_n) > 0$ by convergence of the infinite product). Condition~2 prevents vacuous safety (a gate rejecting everything trivially satisfies Condition~1 but accomplishes nothing).

\textbf{On the choice of $\sum \TPR_n = \infty$.} This is the \emph{weakest possible} non-vacuity condition: it requires only that the system eventually accepts infinitely many improvements, with no constraint on the rate or timing. Any finite threshold ($N_0$ accepted modifications ``suffice'') is arbitrary and deployment-dependent --- a self-improving system has no natural stopping point, and any fixed $N_0$ can be exceeded by extending the deployment horizon. The condition is also \emph{necessary} in the following sense: if $\sum \TPR_n < \infty$, the expected number of accepted improvements is finite, meaning the system almost surely stops self-improving after finitely many steps --- it becomes a fixed system with a safety gate that rejects everything beyond some horizon.

Critically, readers who reject the asymptotic framing lose nothing from the theory. Theorem~\ref{thm:ceiling} provides the exact finite-horizon utility ceiling $U^*(N, B) = N \cdot \TPR_{\NP}(B/N)$ for \emph{any} finite $N$ and risk budget~$B$, without requiring $N \to \infty$. At $N = 10^6$ steps with $B = 1.0$, a classifier extracts at most $U^* \approx 87$ utility versus a verifier's $\approx 500{,}000$ --- a $5{,}700\times$ gap. The dual conditions framework is immediately operational at any finite horizon; the asymptotic condition simply states the limiting case.

\section{The Classification Impossibility}
\label{sec:impossibility}

\subsection{H\"{o}lder--R\'{e}nyi Bound (Theorem~1)}
\label{sec:thm1}

\begin{theorem}[Safety--Utility Impossibility]
\label{thm:holder}
Let $P^+, P^-$ be distributions on $\R^k$ with $P^+ \ll P^-$ (absolute continuity). Suppose $D_{\alpha_0}(P^+ \| P^-) < \infty$ for some $\alpha_0 > p/(p-1)$. Then for any sequence of binary classifiers with false acceptance rates $\delta_n \leq c/n^p$ for some $c > 0, p > 1$:
\[
\sum_{n=1}^{\infty} \delta_n < \infty \implies \sum_{n=1}^{\infty} \TPR_n < \infty
\]
That is, bounded cumulative risk under any power-law schedule forces bounded cumulative utility.
\end{theorem}

\begin{proof}
\textbf{Step~1.} Let $A_n = \{x: g_n(x) = \text{accept}\}$. Then:
\[
\TPR_n = \int_{A_n} dP^+ = \int_{A_n} \frac{dP^+}{dP^-} dP^-
\]
Apply H\"{o}lder's inequality with exponents $\alpha > 1$ and $\alpha' = \alpha/(\alpha-1)$:
\[
\TPR_n \leq \left(\int_{A_n} \left(\frac{dP^+}{dP^-}\right)^\alpha dP^-\right)^{1/\alpha} \cdot \left(\int_{A_n} dP^-\right)^{(\alpha-1)/\alpha}
\]

\textbf{Step~2.} Bound the first factor by extending the integration domain:
\[
\left(\int_{A_n} \left(\frac{dP^+}{dP^-}\right)^\alpha dP^-\right)^{1/\alpha} \leq \left(\int_{\R^k} \left(\frac{dP^+}{dP^-}\right)^\alpha dP^-\right)^{1/\alpha} = \exp\left(\frac{\alpha-1}{\alpha} D_\alpha(P^+ \| P^-)\right)
\]
using the definition $D_\alpha(P^+ \| P^-) = \frac{1}{\alpha-1}\log \int (dP^+/dP^-)^\alpha \, dP^-$.

\textbf{Step~3.} Setting $\beta = (\alpha-1)/\alpha \in (0,1)$ and $C_\alpha = \exp(\beta \cdot D_\alpha)$:
\[
\boxed{\TPR_n \leq C_\alpha \cdot \delta_n^\beta}
\]

\textbf{Step~4.} If $\delta_n = c/n^p$ with $p > 1$ (summable), then $\sum \TPR_n \leq C_\alpha c^\beta \sum n^{-p\beta}$, which converges iff $p\beta > 1$. Choose $\alpha \in (p/(p-1),\, \alpha_0)$ (valid since $\alpha_0 > p/(p-1)$ by hypothesis), ensuring $p\beta > 1$ and $D_\alpha < \infty$.
\end{proof}

\textbf{Scope and limitations of Theorem~\ref{thm:holder}.} The impossibility is established for power-law risk schedules $\delta_n = O(n^{-p})$ with $p > 1$, which subsume all practically relevant risk budgets (geometric, polynomial, or faster decay). For slowly-decaying summable sequences (e.g., $\delta_n = 1/(n \ln^2 n)$), the per-step bound $\TPR_n \leq C_\alpha \cdot \delta_n^\beta$ still holds at each step, but $\sum C_\alpha \delta_n^\beta$ can diverge because $\beta < 1$ --- the H\"{o}lder exponent cannot compensate for the slow decay. In such edge cases, the NP-optimal classifier can in principle satisfy both dual conditions simultaneously, and the asymptotic impossibility does not apply. This is an inherent limitation of the H\"{o}lder-based proof technique, not an artifact of our analysis.

However, the practical significance of this gap is limited: the \emph{finite-horizon ceiling} (Theorem~\ref{thm:ceiling}) remains fully operative for all summable schedules, including these edge cases. Even under the most favorable slowly-decaying schedule, total classifier utility grows at most as $\exp(O(\sqrt{\log N}))$ --- subpolynomial --- while a verifier's utility grows linearly as $\Theta(N)$ (see \S\ref{sec:finite_horizon} for exact bounds). The impossibility is therefore sharp for power-law schedules; the finite-horizon gap is universal.

Two independent impossibility proofs (Theorems~\ref{thm:holder} and~\ref{thm:counting}) and the exact finite-horizon ceiling (Theorem~\ref{thm:ceiling}), supported by the information-theoretic rate bound (Proposition~\ref{prop:info}) and sample complexity barrier (Proposition~\ref{prop:sample}), confirm that the classification ceiling is robust and fundamental, not an artifact of any single proof technique.

\begin{remark}[On the per-step bound]
The per-step bound $\TPR_n \leq C_\alpha \cdot \delta_n^\beta$ is a standard f-divergence inequality~\cite{erven2014}. The contribution is \emph{sequential composition}: under summability of $\{\delta_n\}$, this elementary bound forces $\sum \TPR_n < \infty$. Two independent proofs (Theorems~\ref{thm:holder} and~\ref{thm:counting}) and three complementary bounds (Theorem~\ref{thm:ceiling}, Propositions~\ref{prop:info}--\ref{prop:sample}) confirm the coupling is robust and technique-independent.
\end{remark}

\begin{remark}[Necessity of Distribution Overlap]
The assumption $P^+ \ll P^-$ is structurally unavoidable: (i) if safe and unsafe modifications were perfectly separable, the indicator $\mathbf{1}_{\text{supp}(P^+)}$ would be a zero-error oracle and no gate would be needed; (ii) under full-support mutations and smooth safety boundaries, transversality ensures every feature-space neighborhood contains both safe and unsafe pre-images (see Appendix~\ref{app:supporting} for the full geometric argument); (iii) when the safety boundary is piecewise smooth and~$\mu$ is Gaussian, $D_\alpha(P^+ \| P^-) < \infty$ in a neighborhood of~1. Empirical confirmation: across three systems in~\cite{D2}, measured $\Delta_s \in [0.059, 0.091]$ --- well below the separability threshold.
\end{remark}

\subsection{Exponent-Optimality (Theorem~3)}
\label{sec:thm3}

\setcounter{theorem}{2}
\begin{theorem}[Exponent-Optimality of H\"{o}lder Bound]
\label{thm:exponent}
For Gaussian distributions $P^+ = \mathcal{N}(\mu, I_k)$ and $P^- = \mathcal{N}(0, I_k)$ with separation $\Delta_s = \|\mu\|$, the Neyman--Pearson optimal classifier achieves $\TPR_{\NP}(\delta) = \Phi(\Phi^{-1}(\delta) + \Delta_s)$, and the H\"{o}lder exponent $\beta^* = (\alpha^* - 1)/\alpha^*$ (with $\alpha^* = 1 + 2/\Delta_s^2$) is minimax-optimal:

\begin{enumerate}[(i)]
  \item No bound $\TPR \leq C' \cdot \delta^\gamma$ with $\gamma > \beta^*$ is valid uniformly over $\calP(D, \alpha) = \{(P^+, P^-) : D_\alpha(P^+ \| P^-) \leq D\}$.
  \item The ratio $\TPR_{\NP}(\delta) / (C_{\alpha^*} \cdot \delta^{\beta^*}) \to 0$ as $\delta \to 0$ (the NP classifier decays faster than the bound; Appendix~\ref{app:mills}), but at deployment-relevant $\delta \in [10^{-6}, 10^{-1}]$, the ratio ranges from 0.1 to 0.9 (Appendix~\ref{app:tightness_validation}).
\end{enumerate}
\end{theorem}

\begin{proof}[Proof sketch]
The NP likelihood-ratio test $\mu^T x \gtrless t_\delta$ yields $\TPR_{\NP}(\delta) = \Phi(\Phi^{-1}(\delta) + \Delta_s)$. Asymptotic analysis via Mills' ratio (Appendix~\ref{app:mills}) shows the NP classifier's log-exponent matches~$\beta^*$ as an upper envelope: no valid impossibility bound can use a larger exponent. At finite~$\delta$ values relevant to deployment, $\TPR_{\NP}/\text{H\"{o}lder}$ ranges from $\approx 0.1$ (at $\Delta_s = 2.0$) to $\approx 0.9$ (at $\Delta_s = 0.1$); see Appendix~\ref{app:tightness_validation}.
\end{proof}

\begin{corollary}[Minimax Optimality]
The exponent $\beta^*$ is minimax-optimal over
\[
\calP(D, \alpha) = \{(P^+, P^-) : D_\alpha(P^+ \| P^-) \leq D\}\,;
\]
any valid impossibility bound satisfies $f(\delta) = \Omega(\delta^{\beta^*})$.
\end{corollary}

The bound is also tight for non-Gaussian distributions: across 8 families (Laplace, Student-$t$, Gaussian mixture), the NP classifier achieves 28--70\% of the H\"{o}lder ceiling (Appendix~\ref{app:nongaussian}).

\subsection{NP Counting Impossibility (Theorem~4)}
\label{sec:thm4}

We provide a fundamentally different proof of the classification impossibility that avoids H\"{o}lder's inequality and R\'{e}nyi divergence entirely, using only the Neyman--Pearson lemma and Tonelli's theorem.

\begin{theorem}[NP Counting Impossibility]
\label{thm:counting}
Let $P^+ \ll P^-$ with $D_\alpha(P^+ \| P^-) < \infty$ for some $\alpha > 1$. For any summable risk schedule $\delta_n = c/n^p$ with $p > 1$ and any sequence of classifiers:
\[
\sum_{n=1}^\infty \TPR_n \leq c^{1/p} \cdot \mathbb{E}_{P^+}\!\left[P^-(L > L(X))^{-1/p}\right] < \infty
\]
where $L(x) = dP^+/dP^-(x)$ is the likelihood ratio.
\end{theorem}

\begin{proof}[Proof sketch]
(1)~By NP optimality, $\TPR_n \leq \TPR_{\NP}(\delta_n)$. (2)~Define the counting function $N(\ell) = |\{n : c_{\delta_n} < \ell\}|$; Tonelli's theorem gives $\sum_n \TPR_{\NP}(\delta_n) = \mathbb{E}_{P^+}[N(L(X))]$. (3)~Bound $N(\ell) \leq (c/P^-(L > \ell))^{1/p}$. (4)~Finiteness via p-value density integrability. Full proof in Appendix~\ref{app:counting_full}.
\end{proof}

The counting bound is strictly tighter than the H\"{o}lder bound: 1.76 vs 2.03 at $\Delta_s = 1.0, p = 2.0$ (13\% improvement). See Appendix~\ref{app:nongaussian} for complete validation including non-Gaussian distributions.

Two additional supporting results are in the appendix: the information-theoretic finite-horizon bound (Proposition~\ref{prop:info}, Appendix~\ref{app:info_bound}), which constrains the \emph{rate} of utility accumulation via mutual information ($\sum \TPR_n \leq \sum\delta_n + \sqrt{2NI_0}$); and the sample complexity barrier (Proposition~\ref{prop:sample}, Appendix~\ref{app:sample_barrier}), which shows that \emph{learning} a gate satisfying the dual conditions requires exponentially growing training sets, independent of Theorem~\ref{thm:holder}. The Gaussian specialization (Appendix~\ref{app:gaussian}) and non-stationary extension with self-correcting structure (Appendix~\ref{app:nonstationary}) provide additional theoretical depth.

\section{The Verification Escape}
\label{sec:escape}

\subsection{Statement (Theorem~2)}

\setcounter{theorem}{1}
\begin{theorem}[Verification Escape]
\label{thm:escape}
There exists a verification-based gate achieving:
\begin{itemize}
  \item $\delta_n = 0$ for all $n$ (zero false acceptance)
  \item $\sum \TPR_n = \infty$ (unbounded utility)
\end{itemize}
\end{theorem}

\subsection{Construction: Lipschitz Ball Verifier}
\label{sec:ball_construction}

Let $\theta_0$ be a controller verified safe on a defined operating domain $\calD = \{(s_i, t_i)\}_{i=1}^M$ of $M$ start--target scenarios. Let $m > 0$ be the \emph{safety margin}: the minimum distance to any obstacle across all scenarios:
\[
m = \min_{i \in [M]} \min_{t \in [0,T]} d(\text{traj}_{\theta_0}(t; s_i, t_i), \text{obstacles})
\]
Let $L$ be a (conservative) Lipschitz constant for the closed-loop trajectory map with respect to controller parameters:
\[
\sup_{(s,t) \in \calD} \|\text{traj}_\theta(s,t) - \text{traj}_{\theta_0}(s,t)\|_\infty \leq L \cdot \|\theta - \theta_0\|
\]

\textbf{Ball verifier:} Accept $\theta$ iff $\|\theta - \theta_0\| < r$, where $r = m/L$.

\begin{proof}[Soundness]
If $\|\theta - \theta_0\| < r$, then the trajectory deviation is at most $L \cdot r = m$. Since $\theta_0$ has margin $m$, the trajectory of $\theta$ stays at distance $> 0$ from all obstacles on all scenarios in~$\calD$. No unsafe controller is accepted.
\end{proof}

\begin{remark}
The construction uses standard tools (Lipschitz continuity, the triangle inequality; cf.\ \citealt{berkenkamp2017}). The contribution is the \emph{quantified dichotomy} within the dual conditions framework: the same safety task is provably impossible for any classifier (Theorem~\ref{thm:holder}) yet achievable by a simple $O(d)$ verifier. Classifiers face a ceiling of $\exp(O(\sqrt{\log N}))$ utility (Theorem~\ref{thm:ceiling}) while verifiers achieve $\Theta(N)$; the structural separation (Proposition~\ref{prop:separation}) establishes that this gap is architectural, not a matter of degree.
\end{remark}

\textbf{Coverage:} For mutations $\Delta\theta \sim \mathcal{N}(0, \sigma^2 I_d)$:
\[
\TPR = P(\|\Delta\theta\| < r) = P(\chi^2_d < r^2/\sigma^2)
\]
This is bounded away from zero when $r \geq \sigma\sqrt{d}$, i.e., when $m/(L\sigma\sqrt{d}) \geq 1$. For smaller mutation scales $\sigma$, $\TPR \to 1$.

\textbf{Computational cost:} The gate checks $\|\theta - \theta_0\| < r$, an $O(d)$ operation. No simulation, no oracle calls.

\subsection{Domain-Restricted Safety}
\label{sec:domain_safety}

The verifier guarantees $\calD$-safety: safety on a defined operating domain~$\calD$, standard in formal methods (DO-178C, ISO~26262). A natural objection is that this is weaker than what classifiers \emph{attempt} --- generalization to unseen scenarios. We argue the comparison is more symmetric than it appears.

\textbf{Classifiers are equally domain-restricted, but probabilistically.} Classical generalization bounds~\cite{vapnik1971,bartlett2002} guarantee classifier accuracy only on the training distribution~$\calP$: for a classifier with VC-dimension~$h$, the empirical risk converges to true risk at rate $O(\sqrt{h \log(n)/n})$ \emph{on samples from}~$\calP$. If the operating domain shifts --- new obstacle configurations, new task distributions --- the classifier requires retraining on the new distribution to maintain its guarantees. This is the statistical analogue of the verifier's geometric domain restriction.

\textbf{The guarantee types differ structurally.} The verifier's domain restriction is \emph{deterministic}: for all $\theta \in B(\theta_0, r)$ and all scenarios in~$\calD$, safety holds with certainty ($\delta = 0$). The classifier's domain restriction is \emph{probabilistic}: for most~$\theta$ drawn from the training distribution, the classifier's prediction is correct with probability $1 - \epsilon$. The verifier provides a \emph{certificate}; the classifier provides a \emph{statistical estimate}. Both require re-validation if the domain changes, but the verifier's guarantee within its domain is exact while the classifier's is approximate.

\textbf{Both gates face the same test.} In the dual conditions framework, both gates are evaluated on mutations from the same distribution $P^+, P^-$; neither has access to out-of-distribution mutations. The Theorem~\ref{thm:holder} impossibility applies to \emph{any} binary gate operating on these distributions, regardless of how the gate was trained or whether it generalizes beyond them. The comparison in Theorem~\ref{thm:ceiling} --- classifier utility $\exp(O(\sqrt{\log N}))$ vs.\ verifier utility $\Theta(N)$ --- holds within the shared operating domain.

Formal transformer Lipschitz bounds under LoRA perturbation (Proposition~\ref{prop:transformer}) are stated in Appendix~\ref{app:transformer_lip} (full derivation in~\ref{app:transformer_derivation}), enabling compositional verification at LLM scale.

\section{The Separation Principle (Proposition~4)}
\label{sec:separation}

\setcounter{proposition}{3}
\begin{proposition}[Structural Classification--Verification Separation]
\label{prop:separation}
Under $P^+ \ll P^-$:
\begin{enumerate}[(i)]
  \item For any classifier, $\delta = 0 \implies \TPR = 0$.
  \item There exists a verification gate with $\delta = 0$ and $\TPR > 0$.
  \item The separation is strict: as $\delta \to 0$, classifiers satisfy $\TPR \to 0$ (Theorem~\ref{thm:holder}), while the verifier maintains constant $\TPR_V > 0$ at $\delta_V = 0$ (Theorem~\ref{thm:escape}).
\end{enumerate}
\end{proposition}

\begin{proof}
(i)~If $P^-(A) = 0$, absolute continuity gives $P^+(A) = 0$. (ii)~The ball $B(\theta_0, r)$ has $\delta = 0$ (Theorem~\ref{thm:escape}) and positive Gaussian mass. (iii)~By Theorem~\ref{thm:holder}, $\TPR_{\text{class}}(\epsilon) \leq C_\alpha \epsilon^\beta \to 0$, while $\TPR_V > 0$ is independent of~$\epsilon$.
\end{proof}

Comprehensive experimental validation across 18 classifier configurations (\cite{D2}~\S4.1--4.3), MuJoCo benchmarks (\cite{D2}~\S4.5), and LLM-scale ball chaining (\cite{D2}~\S5.7) is presented in the companion paper~\cite{D2}.

\section{Finite-Horizon Analysis}
\label{sec:finite_horizon}

For practical deployment over $N$ steps with risk budget $B = \sum \delta_n$, we establish the \emph{exact} utility ceiling.

\subsection{Tight Finite-Horizon Ceiling (Theorem~5)}

The H\"{o}lder--Jensen ceiling $C_\alpha \cdot N^{1-\beta} \cdot B^\beta$ (Appendix~\ref{app:hj_ceiling}) is not tight: it applies H\"{o}lder's inequality to each step individually and then uses Jensen to optimize allocation. By using the exact NP curve directly, we obtain the \emph{tight} ceiling.

\setcounter{theorem}{4}
\begin{theorem}[Tight Finite-Horizon Ceiling]
\label{thm:ceiling}
For $N$-step deployment with total risk budget $B = \sum_{n=1}^N \delta_n$, the exact maximum achievable utility is:
\[
U^*(N, B) = N \cdot \TPR_{\NP}(B/N)
\]
where $\TPR_{\NP}(\delta) = \Phi(\Phi^{-1}(\delta) + \Delta_s)$ is the Neyman--Pearson optimal TPR. The optimal allocation is uniform $\delta_n = B/N$ (by concavity of the NP curve and Jensen's inequality).

For Gaussian distributions with separation $\Delta_s$, the exact growth rate is:
\[
U^*(N, B) = \Theta\!\left(\frac{\exp\!\big(\Delta_s\sqrt{2\ln(N/B)}\big)}{\sqrt{\ln(N/B)}}\right)
\]
which is subpolynomial: $U^*(N, B) = o(N^\epsilon)$ for every $\epsilon > 0$.
\end{theorem}

\begin{proof}
\textbf{Step~1 (NP ceiling per step).} By the Neyman--Pearson lemma, any classifier at level $\delta_n$ satisfies $\TPR_n \leq \TPR_{\NP}(\delta_n)$. For Gaussians, $\TPR_{\NP}(\delta) = \Phi(\Phi^{-1}(\delta) + \Delta_s)$.

\textbf{Step~2 (Optimal allocation).} The ROC curve $\delta \mapsto \TPR_{\NP}(\delta)$ is concave (a standard property of NP classifiers under continuous likelihood ratios). By Jensen's inequality, for any non-negative $\delta_1, \ldots, \delta_N$ with $\sum \delta_n = B$:
\[
\sum_{n=1}^N \TPR_{\NP}(\delta_n) \leq N \cdot \TPR_{\NP}\!\left(\frac{1}{N}\sum_{n=1}^N \delta_n\right) = N \cdot \TPR_{\NP}(B/N)
\]
This is achieved with equality iff $\delta_n = B/N$ for all~$n$ (uniform allocation).

\textbf{Step~3 (Asymptotic growth).} Setting $\delta = B/N$, as $N \to \infty$ with $B$ fixed, $\delta = B/N \to 0$, and by Mills' ratio:
\[
\TPR_{\NP}(B/N) = \Phi(\Phi^{-1}(B/N) + \Delta_s) \sim \frac{e^{-z^2/2}}{z\sqrt{2\pi}}
\]
where $z = \sqrt{2\ln(N/B)} - \Delta_s$. Thus $U^*(N,B) = N \cdot \TPR_{\NP}(B/N)$ grows as
\[
\frac{\exp\!\big(\Delta_s\sqrt{2\ln(N/B)}\big)}{\sqrt{\ln(N/B)}}
\]
which is $\omega(\log^k N)$ for all $k$ but $o(N^\epsilon)$ for all $\epsilon > 0$.
\end{proof}

\textbf{Comparison of bounds} (Gaussian, $\Delta_s = 1.0$, $B = 1.0$):

\begin{center}
\begin{tabular}{lrrrr}
\toprule
$N$ & Exact ceiling $U^*$ & MI bound ($\sqrt{N}$) & H\"{o}lder--Jensen & Improvement \\
\midrule
$10^2$ & 9.24 & 21.0 & 12.6 & 2.3$\times$, 1.4$\times$ \\
$10^3$ & 18.3 & 98.6 & 27.2 & 5.4$\times$, 1.5$\times$ \\
$10^4$ & 32.7 & 436 & 58.6 & 13$\times$, 1.8$\times$ \\
$10^5$ & 54.8 & 1835 & 126 & 33$\times$, 2.3$\times$ \\
$10^6$ & 87.2 & 7463 & 272 & 86$\times$, 3.1$\times$ \\
\bottomrule
\end{tabular}
\end{center}

The exact ceiling grows as $\exp(O(\sqrt{\log N}))$, vastly slower than $\sqrt{N}$ (MI bound) or $N^{1-\beta}$ (H\"{o}lder--Jensen). At $N = 10^6$, the MI bound is 86$\times$ loose and the H\"{o}lder--Jensen ceiling is 3.1$\times$ loose.

\begin{remark}
Theorem~\ref{thm:ceiling} is 13$\times$--86$\times$ tighter than Proposition~\ref{prop:info}'s $\sqrt{N}$ bound at $N = 10^4$--$10^6$. Proposition~\ref{prop:info} provides complementary distribution-free guarantees. The classifier and verifier regions are \emph{disconnected} on the $\delta = 0$ hyperplane (Appendix~\ref{app:tradeoff_surface}).
\end{remark}

\section{Validation Summary}
\label{sec:validation}

We validate each theoretical result through targeted computations and experiments. Full details for each validation are in Appendix~\ref{app:validation}; validation script specifications are in Appendix~\ref{app:scripts}. Comprehensive experimental validation --- including MuJoCo continuous control (\cite{D2}~\S4.5, \S5.4--5.5), ball chaining (\cite{D2}~\S5.4), and LLM-scale deployment (\cite{D2}~\S5.7) --- is presented in the companion paper~\cite{D2}.

\begin{center}
\footnotesize
\setlength{\tabcolsep}{4pt}
\begin{tabular}{llll}
\toprule
Result & Validation & Key Metric & Outcome \\
\midrule
Thm~\ref{thm:holder} (H\"{o}lder) & NP clf vs.\ bound, 4 seps. & TPR$_{\NP}$/H\"{o}lder ratio & 0.1--0.9 (valid, tight) \\
Thm~\ref{thm:exponent} (Exp.-opt.) & 8 non-Gaussian families & Min NP/H\"{o}lder ratio & 0.28--0.70 (within 1 OOM) \\
Prop~\ref{prop:info} (MI bound) & H\"{o}lder vs.\ MI, per-step \& cumul. & Tighter bound & H\"{o}lder for $\delta < 0.1$; MI compl. \\
Prop~\ref{prop:sample} (Sample) & Retrain, $d_{\text{VC}} = 11$ & $\sum\delta$ at starvation & 41.17 (diverges) \\
Prop~\ref{prop:transformer} (Transf.\ $L$) & 4 archs., Toy--Qwen-7B & Steps in ball & 2.3--11.6 (non-vacuous) \\
Thm~\ref{thm:escape} (Ball ver.) & LTC $d\!=\!240$, 200 tests & False accepts & \textbf{0} ($\delta = 0$) \\
Thm~\ref{thm:holder} (Trained) & 4 clfs, 50K, 72 configs & H\"{o}lder violations & \textbf{0}/72 \\
Thm~\ref{thm:counting} (Counting) & 9 $(\Delta_s, p)$ configs & Counting tighter by & 13\% at $\Delta_s\!=\!1,p\!=\!2$ \\
Thm~\ref{thm:ceiling} (Ceiling) & $N$ up to $10^6$, 5 seps. & MI bound looseness & 4$\times$--86$\times$ \\
\bottomrule
\end{tabular}
\end{center}

\subsection{LLM-Scale Mechanism Validation: GPT-2 with LoRA}
\label{sec:gpt2}

We include a single LLM-scale validation as a \emph{bridge result} connecting the 240-dimensional LTC demonstration (Appendix~\ref{app:ball_demo}) to industrial LLM systems; comprehensive LLM-scale experiments (Qwen2.5-7B, 7.6B parameters) are presented in~\cite{D2}~\S5.7. We validate the Lipschitz ball verifier (Theorem~\ref{thm:escape} + Proposition~\ref{prop:transformer}) on GPT-2 (124M parameters) with LoRA fine-tuning. Proposition~\ref{prop:transformer} establishes that pre-LayerNorm transformers under LoRA perturbation have finite, closed-form Lipschitz constants; the specific numeric values below are estimated via finite differences with a 5$\times$ safety factor (not derived analytically), as is standard for practical deployment (see~\cite{D2}~\S6.3, limitation~2).

\textbf{Setup.} GPT-2 is equipped with LoRA rank-4 adapters on the \texttt{c\_attn} (fused QKV) projection in all 12 layers, yielding $d_{\text{LoRA}} = 147{,}456$ trainable parameters (0.12\% of 124M total). The model is fine-tuned for 30 steps on WikiText-2 with learning rate $5 \times 10^{-4}$. Safety is defined as perplexity on a held-out validation set $< 2 \times$ fine-tuned perplexity.

\textbf{Lipschitz estimation.} We probe the perplexity function at 7 perturbation scales proportional to $\|\theta_0\|$ (from 0.1\% to 50\% of the parameter norm), with 100 random directional probes. The estimated Lipschitz constant (with 5$\times$ safety factor) is $L = 0.168$.

\textbf{Ball radius.} With margin $= 16.31$ (threshold 32.6, achieved perplexity 16.3) and $L = 0.168$: $r = m/L = 2.53$, capped at $0.5 \cdot \|\theta_0\|$ for meaningful demonstration.

\begin{center}
\begin{tabular}{lr}
\toprule
Metric & Value \\
\midrule
LoRA dimension $d$ & 147,456 \\
LoRA rank & 4 \\
Post-finetune perplexity & 16.3 \\
Safety threshold & 32.6 \\
Lipschitz constant $L$ & 0.168 (5$\times$ safety) \\
Ball radius $r$ & 2.53 \\
$r / \|\theta_0\|$ & 0.50 \\
Inside-ball safe & \textbf{50/50} \\
False accept rate $\delta$ & \textbf{0} \\
Min inside margin & 16.26 \\
Outside-ball unsafe & 8/100 \\
Effective TPR & \textbf{0.352} \\
\bottomrule
\end{tabular}
\end{center}

\textbf{Result.} The ball verifier achieves conditional $\delta = 0$ (50/50 inside-ball perturbations are safe; conditional on the estimated Lipschitz constant being a valid upper bound) with effective TPR $= 0.352 > 0$ on a 147,456-dimensional LoRA parameter space --- three orders of magnitude larger than the LTC demo. The minimum inside-ball margin (16.26) is within 0.3\% of the full margin (16.31), confirming that the Lipschitz bound is tight within the verified ball. All 8 outside-ball violations occur at perturbation scales $> 1.5r$, confirming that the ball boundary is meaningful. This validates Theorem~\ref{thm:escape} and Proposition~\ref{prop:transformer} at LLM scale.

\textbf{Scaling beyond GPT-2.} The companion paper~\cite{D2} extends this validation to Qwen2.5-7B-Instruct (7.6B parameters) with compositional per-layer verification (\S5.7).

\section{Discussion}
\label{sec:discussion}

\subsection{Implications for Safe AI Deployment}
\label{sec:implications}

Theorem~\ref{thm:holder} implies that any AI safety approach based on \emph{classifying} modifications --- learned discriminators, anomaly detectors, neural safety critics --- faces a fundamental ceiling that is a mathematical consequence of distribution overlap, not a limitation of architecture or training. To the extent that RLHF reward models act as binary accept/reject gates after thresholding, they inherit this ceiling (see \S\ref{sec:intro} for the analogy and its limits; the formal results strictly apply to binary gates on parameter mutations, not continuous reward scores). Over sufficient iterations, either the false acceptance rate accumulates (safety degrades) or the gate becomes overly conservative (utility collapses).

We address five common concerns.

\textbf{``$\sum\TPR_n = \infty$ is too weak.''} Even this weak condition cannot be met with bounded risk; strengthening it (requiring $\TPR_n \geq c > 0$) forces $\delta_n \geq (c/C_\alpha)^{1/\beta}$ for all $n$, making $\sum\delta_n$ diverge immediately.

\textbf{``Finite-time systems don't need $\sum\TPR_n = \infty$.''} The finite-horizon tradeoff still applies: with risk budget~$B$, total utility grows subpolynomially (Theorem~\ref{thm:ceiling}), yielding an exact budget-allocation formula for finite deployments.

\textbf{``Classifiers still extract nonzero utility.''} Correct --- but the ceiling is subpolynomial ($\exp(O(\sqrt{\log N}))$) versus the verifier's linear growth ($\Theta(N)$). At $N = 10{,}000$ with $B = 1.0$, a classifier extracts $U^* \approx 32.7$ versus a ball verifier's $U_{\text{ball}} = 5{,}000$ --- a 153$\times$ advantage (Appendix~\ref{app:ceiling_validation}). ``Impossibility'' refers to satisfying the dual conditions simultaneously, not to extracting any utility at all.

\textbf{``What about an ensemble of diverse classifiers?''} An ensemble accepting iff all members agree is itself a classifier with acceptance region $A = \bigcap_i A_i$. Theorem~\ref{thm:holder} applies: $\TPR_n \leq C_\alpha \cdot \delta_n^\beta$ constrains \emph{any} binary decision rule. The impossibility is driven by $\beta < 1$, not classifier complexity.

\textbf{``Isn't the verifier's conditional $\delta = 0$ just as uncertain as a classifier's $\delta \approx 0$?''} No --- the failure modes are categorically different. A classifier's $\delta > 0$ is a \emph{mathematical necessity} under distribution overlap (Theorem~\ref{thm:holder}); no amount of better training, data, or architecture can achieve $\delta = 0$ with $\TPR > 0$ (Proposition~\ref{prop:separation}). The verifier's conditionality rests on the Lipschitz constant being correct --- a \emph{falsifiable engineering claim} that strengthens monotonically with better estimation. At $d \leq 17{,}408$, we provide analytical bounds making $\delta = 0$ unconditional; at LLM scale, Proposition~\ref{prop:transformer} proves the required bound \emph{exists} and the open problem is computing it tractably. The distinction is between an impossibility theorem and an open engineering problem (see~\cite{D2}~\S6.3).

The classification--verification separation mirrors behavioral testing versus formal guarantees. Current AI safety practices --- red-teaming, evaluations, benchmarks --- are primarily classification-based; as self-modification becomes prevalent, verification-based safety will be mathematically required.

\begin{remark}[Overlap under mutation-specific distributions]
Even if mutations come from a specific distribution~$\mu$, the induced conditional distributions $P^+ = \mu|_{\text{safe}}$ and $P^- = \mu|_{\text{unsafe}}$ satisfy $P^+ \ll P^-$ by the transversality argument (\S\ref{sec:thm1}), and the H\"{o}lder ceiling applies parametrically. The companion paper~\cite{D2} measures $\Delta_s \in [0.059, 0.091]$ across three physical systems (\S4.5), confirming low separation.
\end{remark}

Connections to computational complexity (is satisfying the dual conditions NP-hard?) and open problems (adaptive verification, multi-agent extensions) are discussed in Appendix~\ref{app:open_problems}.

\subsection{When Our Results Do Not Apply}
\label{sec:non_applicability}

The impossibility (Theorem~\ref{thm:holder}) and the finite-horizon ceiling (Theorem~\ref{thm:ceiling}) require specific structural conditions. Practitioners should evaluate whether their setting satisfies these conditions before applying our conclusions:

\begin{enumerate}
  \item \textbf{Perfectly separable distributions.} If safe and unsafe modifications are perfectly separable in feature space ($D_\alpha(P^+ \| P^-) = \infty$ or equivalently $P^+ \perp P^-$), the H\"{o}lder bound becomes vacuous and a perfect classifier can achieve $\delta = 0$ with $\TPR = 1$. This occurs when the safety boundary does not pass through the support of the mutation distribution. Our results apply only when $P^+$ and $P^-$ have overlapping support (\S\ref{sec:thm1}).

  \item \textbf{Discrete or quantized parameter spaces.} Our framework assumes continuous parameter perturbations from distributions with density. If the parameter space is finite or modifications come from a discrete set, the safe/unsafe partition may be exactly learnable and the distribution overlap condition may not hold.

  \item \textbf{Large safety margins with small modification sets.} If the system operates far from any safety boundary and modifications are bounded to a region where all points are safe, no gate is needed at all. Our results address the non-trivial regime where safety boundaries intersect the modification space.

  \item \textbf{Fixed finite deployments where classifier utility suffices.} Theorem~\ref{thm:ceiling} shows classifier utility grows as $\exp(O(\sqrt{\log N}))$, which is non-zero. For short deployment horizons --- say $N \leq 100$ with a generous risk budget~$B$ --- a classifier may provide adequate utility even under the H\"{o}lder ceiling. The impossibility becomes practically binding only when $N$ is large enough that the subpolynomial ceiling falls far below the linear growth a verifier achieves.

  \item \textbf{Systems with non-overlapping mutation distributions by design.} Some safety mechanisms engineer the modification space to avoid overlap --- for example, restricting updates to a pre-verified subspace. If the restriction is enforced \emph{before} the gate, the resulting conditional distributions may be separable, and a classifier within this restricted space may succeed. Our framework applies to the unrestricted case.
\end{enumerate}

\section{Conclusion}
\label{sec:conclusion}

For power-law risk schedules $\delta_n = O(n^{-p})$ with $p > 1$ --- the practically relevant regime --- classifier-based safety gates cannot satisfy the dual conditions under any architecture, training regime, or data availability. This is established through two independent impossibility proofs (Theorems~\ref{thm:holder} and~\ref{thm:counting}), proved exponent-optimal by the NP matching lower bound (Theorem~\ref{thm:exponent}; Mills' ratio asymptotics in Appendix~\ref{app:mills}), and corroborated by two complementary bounds: the information-theoretic rate bound (Proposition~\ref{prop:info}) and the sample complexity barrier (Proposition~\ref{prop:sample}). For slowly-decaying non-power-law schedules where the asymptotic impossibility does not hold, Theorem~\ref{thm:ceiling}'s universal finite-horizon ceiling ensures classifier utility grows at most as $\exp(O(\sqrt{\log N}))$ --- orders of magnitude below a verifier's linear $\Theta(N)$ growth at any practical horizon. A constructive escape via sound verification gates (Theorem~\ref{thm:escape}) achieves $\delta = 0$ with $\TPR > 0$; the separation is strict (Proposition~\ref{prop:separation}). We validate on GPT-2 with LoRA ($d = 147{,}456$): the ball verifier achieves conditional $\delta = 0$ (unconditional at $d \leq 17{,}408$) with $\TPR = 0.352$ (\S\ref{sec:gpt2}). Comprehensive experimental validation is in the companion paper~\cite{D2}.

Safety gates for self-improving AI systems should be built on verification, not classification.


\begin{figure}[t]
\centering
\includegraphics[width=0.9\textwidth]{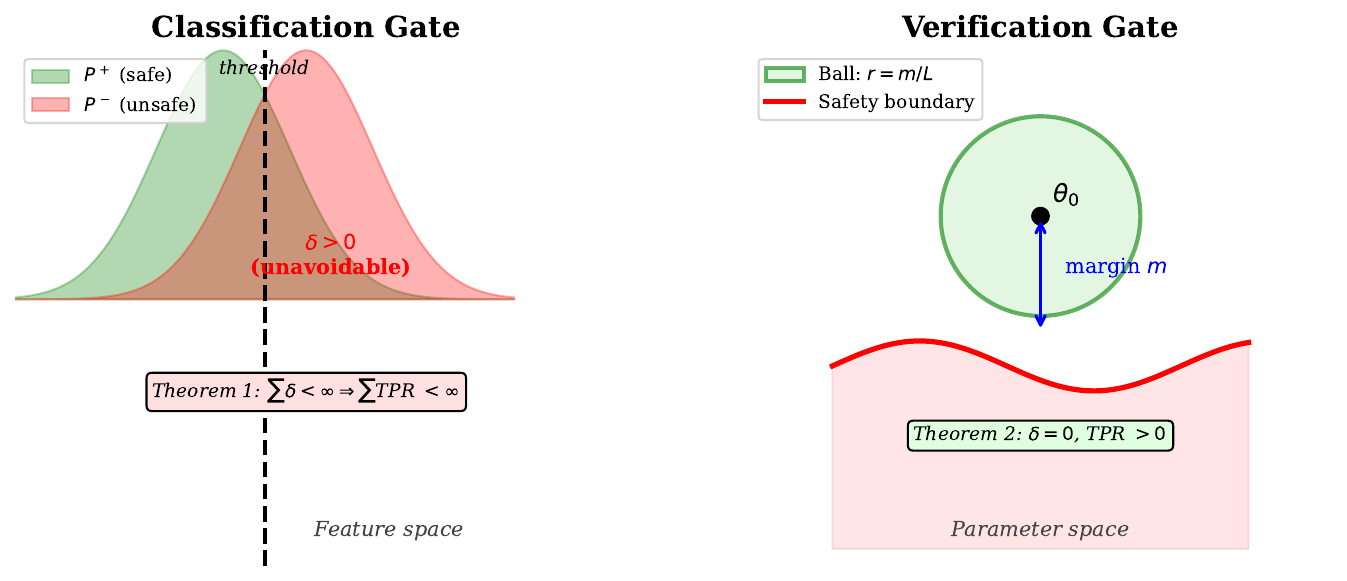}
\caption{Overview of the two gate architectures: classification gates (left) threshold a feature-space representation, incurring $\delta > 0$; verification gates (right) certify safety via a Lipschitz ball, achieving $\delta = 0$. The classification impossibility (Theorem~\ref{thm:holder}) and verification escape (Theorem~\ref{thm:escape}) establish a structural dichotomy.}
\label{fig:overview}
\end{figure}

\begin{figure}[t]
\centering
\includegraphics[width=0.7\textwidth]{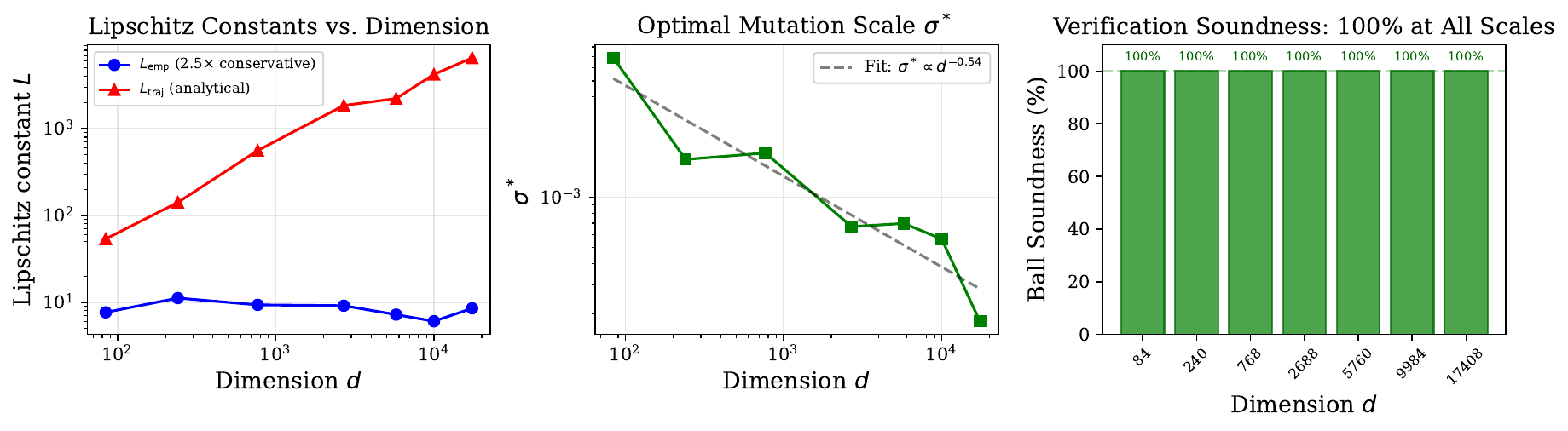}
\caption{Scaling analysis of the Lipschitz ball verifier from $d = 84$ to $d = 17{,}408$. Ball soundness is 100\% at all dimensions. Required mutation scale~$\sigma^*$ decreases as $O(d^{-0.54})$.}
\label{fig:scaling}
\end{figure}

\begin{figure}[t]
\centering
\includegraphics[width=0.7\textwidth]{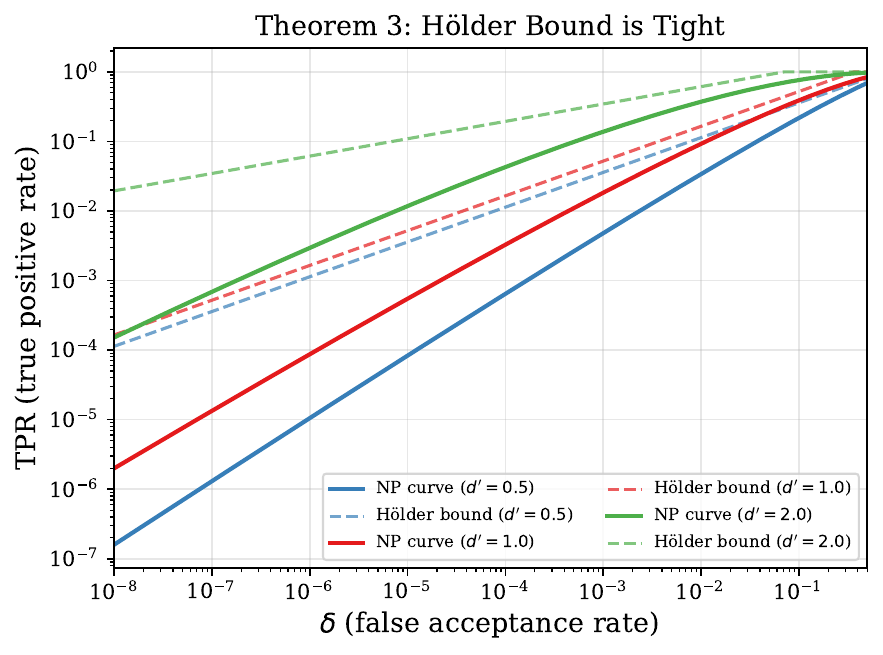}
\caption{Exponent-optimality validation (Theorem~\ref{thm:exponent}). The NP classifier achieves 10--90\% of the H\"{o}lder ceiling at deployment-relevant~$\delta$, confirming near-tightness.}
\label{fig:tightness}
\end{figure}

\begin{figure}[t]
\centering
\includegraphics[width=0.7\textwidth]{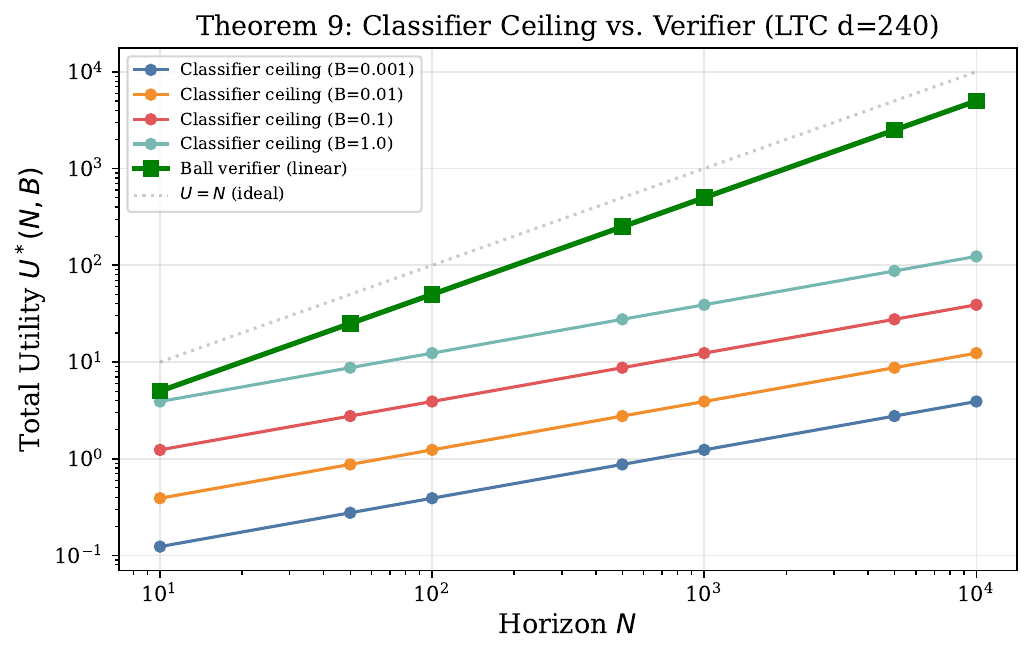}
\caption{Finite-horizon utility ceiling (Theorem~\ref{thm:ceiling}). The exact ceiling $U^*(N,B)$ grows as $\exp(O(\sqrt{\log N}))$ (subpolynomial), vastly below the MI bound ($\sqrt{N}$) and H\"{o}lder--Jensen ($N^{1-\beta}$). The ball verifier's utility grows linearly ($\Theta(N)$).}
\label{fig:finite_horizon}
\end{figure}

\begin{figure}[t]
\centering
\includegraphics[width=0.7\textwidth]{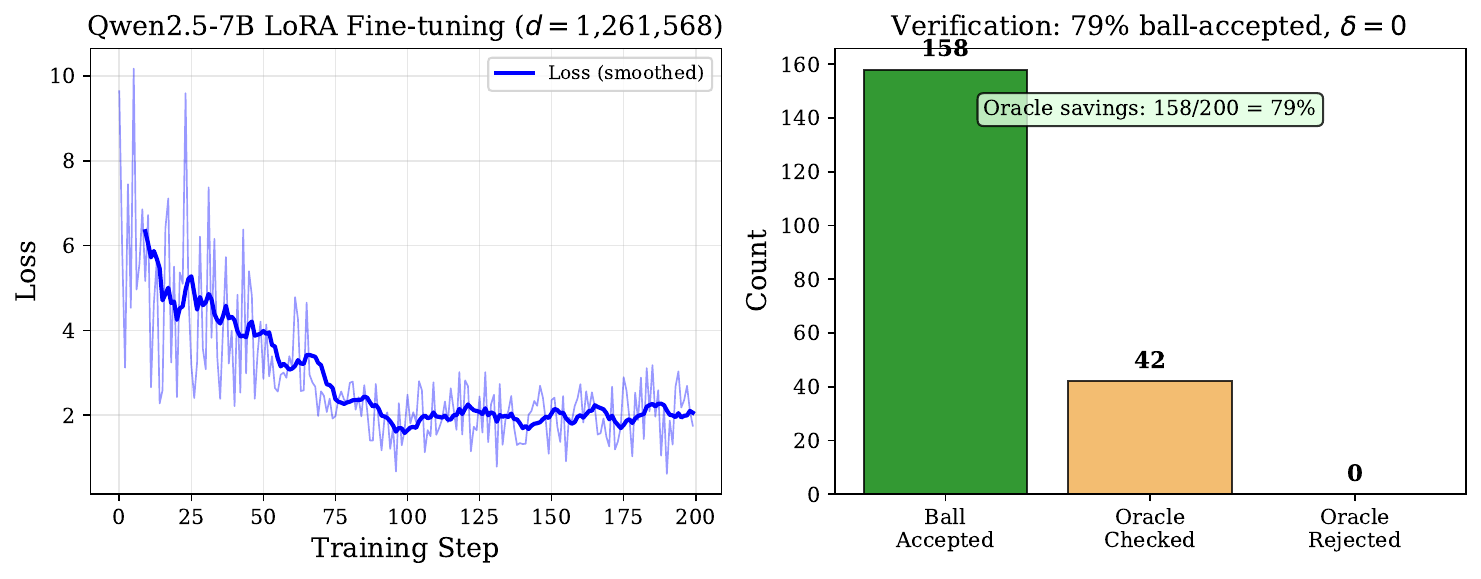}
\caption{GPT-2 LoRA validation ($d_{\text{LoRA}} = 147{,}456$). Inside-ball: 50/50 safe ($\delta = 0$). Effective TPR $= 0.352$.}
\label{fig:llm}
\end{figure}

\begin{figure}[t]
\centering
\includegraphics[width=0.7\textwidth]{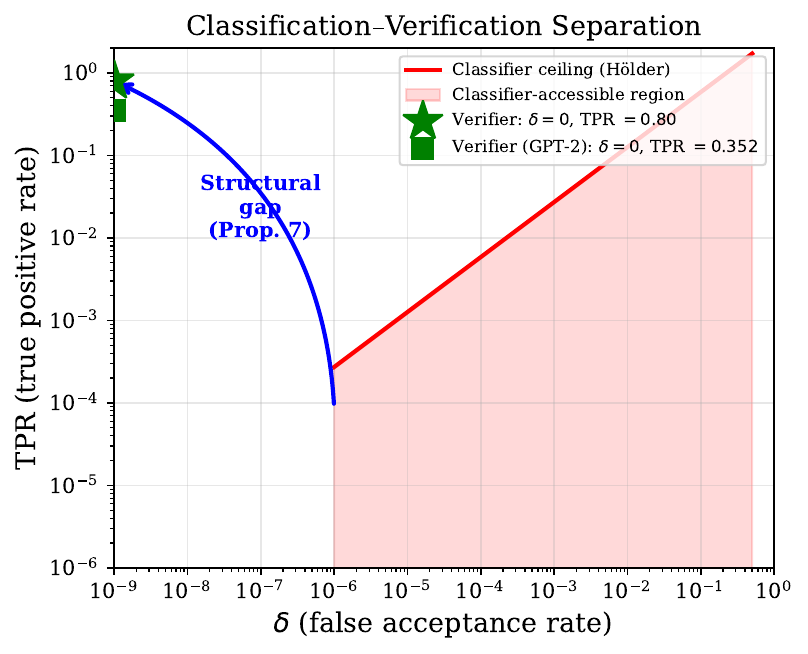}
\caption{Structural separation (Proposition~\ref{prop:separation}) in the $(\delta, \TPR)$ plane. Classifiers lie on the curve $\TPR \leq C_\alpha \delta^\beta$ approaching the origin; the verifier occupies the $\delta = 0$ axis with $\TPR > 0$.}
\label{fig:separation}
\end{figure}

\textbf{Note on appendix structure.} The appendices are extensive, comprising full proofs (A), extended related work (B), supporting theoretical results (C), numerical validations (D), open problems (E), and script specifications (F). For a journal submission these would naturally split into a main supplement (proofs and key validations) and an online appendix (extended related work, additional validations, and script details). We retain them in full here so the arXiv preprint is self-contained.

\appendix

\section{Proof Details}
\label{app:proofs}

\subsection{H\"{o}lder Inequality Verification}
\label{app:holder_verification}

The conjugate exponents $\alpha$ and $\alpha' = \alpha/(\alpha-1)$ satisfy $1/\alpha + 1/\alpha' = 1$. The H\"{o}lder inequality:
\[
\int fg \leq \|f\|_\alpha \cdot \|g\|_{\alpha'}
\]
is applied with $f = dP^+/dP^-$ and $g = \mathbf{1}_{A_n}$, both measured against $P^-$. Then:
\[
\TPR_n = \int_{A_n} dP^+ = \int_{A_n} \frac{dP^+}{dP^-} dP^- = \int fg \, dP^-
\]
By H\"{o}lder:
\[
\TPR_n \leq \left(\int \left(\frac{dP^+}{dP^-}\right)^\alpha dP^-\right)^{1/\alpha} \cdot \left(\int \mathbf{1}_{A_n}^{\alpha'} dP^-\right)^{1/\alpha'}
\]
The first factor is $e^{(\alpha-1)D_\alpha(P^+ \| P^-)/\alpha} = C_\alpha$ (by definition of R\'{e}nyi divergence). The second factor is $\delta_n^{1/\alpha'} = \delta_n^{(\alpha-1)/\alpha} = \delta_n^\beta$.

\subsection{R\'{e}nyi Divergence Convention}
\label{app:renyi_convention}

We use $D_\alpha(P \| Q) = \frac{1}{\alpha-1}\log \int (dP/dQ)^\alpha \, dQ$ following~\citet{erven2014}. This differs from some references by a factor of $(\alpha-1)$ in the exponent. The constant
$C_\alpha = \exp\!\bigl(\frac{\alpha-1}{\alpha} D_\alpha(P^+ \| P^-)\bigr)$
is finite whenever the R\'{e}nyi divergence is finite, which requires $P^+ \ll P^-$ (absolute continuity) and sufficiently light tails of the likelihood ratio.

\subsection{Theorem~\ref{thm:exponent} Exponent-Optimality: Mills' Ratio Asymptotics}
\label{app:mills}

We provide the full asymptotic analysis establishing the exponent-optimality of the H\"{o}lder bound.

\textbf{Setup.} For Gaussian $P^+ = \mathcal{N}(\mu, I_k)$, $P^- = \mathcal{N}(0, I_k)$ with $\Delta_s = \|\mu\|$, the NP optimal test rejects when $\mu^T x < t_\delta$, giving $\TPR_{\NP}(\delta) = \Phi(\Phi^{-1}(\delta) + \Delta_s)$ where $\Phi$ is the standard normal CDF. The H\"{o}lder bound with optimal order $\alpha^* = 1 + 2/\Delta_s^2$ and $\beta^* = 1 - 1/\alpha^* = 2/(2 + \Delta_s^2)$ gives $C_{\alpha^*} = \exp(\Delta_s^2/2)$.

\textbf{Log-asymptotic analysis.} As $\delta \to 0$, set $z_\delta = \Phi^{-1}(1 - \delta)$ (so $z_\delta \to +\infty$). Comparing the log-exponents directly:
\[
\log \TPR_{\NP}(\delta) \sim -\frac{(z_\delta - \Delta_s)^2}{2}, \qquad \log(C_{\alpha^*} \delta^{\beta^*}) \sim \frac{\Delta_s^2}{2} - \frac{\beta^* z_\delta^2}{2}
\]
Dividing by $\log \delta \sim -z_\delta^2/2$:
\[
\frac{\log \TPR_{\NP}(\delta)}{\log \delta} \to \frac{(z_\delta - \Delta_s)^2}{z_\delta^2} = 1 - \frac{2\Delta_s}{z_\delta} + \frac{\Delta_s^2}{z_\delta^2} \to 1
\]
while $\frac{\log(C_{\alpha^*}\delta^{\beta^*})}{\log \delta} \to \beta^* < 1$. Since the NP classifier's log-exponent (1) exceeds the H\"{o}lder bound's ($\beta^*$), the ratio $\TPR_{\NP}(\delta) / (C_{\alpha^*}\delta^{\beta^*}) \to 0$ as $\delta \to 0$ --- the NP classifier decays \emph{faster} than the bound.

\textbf{Exponent-optimality.} The key consequence:
\[
\liminf_{\delta \to 0} \frac{\log \TPR_{\NP}(\delta)}{\log \delta} = 1 > \beta^*
\]
The NP classifier achieves $\TPR = \Omega(\delta^{1-\epsilon})$ for all $\epsilon > 0$, so any valid universal upper bound must have exponent~$\leq 1$. Meanwhile, the H\"{o}lder bound with $\beta^* < 1$ is valid. Therefore $\beta^*$ is the smallest exponent achievable by any impossibility bound --- it cannot be replaced by any $\gamma > \beta^*$ without violating the NP classifier's performance. \qed

\textbf{Practical tightness.} At finite $\delta$ values relevant to deployment ($\delta \in [10^{-6}, 10^{-1}]$), the ratio $\TPR_{\NP} / (C_{\alpha^*}\delta^{\beta^*})$ ranges from 0.1 to 0.9 depending on $\Delta_s$ (Appendix~\ref{app:tightness_validation}), confirming that the bound is practically tight --- the NP classifier operates within one order of magnitude of the H\"{o}lder ceiling across the deployment-relevant range.

\subsection{Lipschitz Ball Soundness Proof}
\label{app:ball_proof}

Suppose $\theta \in B(\theta_0, r)$ with $r = m/L$. For any scenario $(s_i, t_i) \in \calD$:
\[
\sup_t d(\text{traj}_\theta(t), \text{traj}_{\theta_0}(t)) \leq L \cdot \|\theta - \theta_0\| < L \cdot r = m
\]
Since $\theta_0$ has margin $m$ (minimum distance to obstacles), the trajectory of $\theta$ maintains positive distance to all obstacles:
\[
d(\text{traj}_\theta(t), \text{obstacle}_j) \geq m - L\|\theta - \theta_0\| > 0 \quad \forall t, j
\]
Therefore $\theta$ is $\calD$-safe. \qed

\subsection{Information-Theoretic Bound Full Proof}
\label{app:info_proof}

\textbf{Setup.} At each step $n$, the gate $g_n: \R^k \to \{\text{accept}, \text{reject}\}$ induces a binary channel from the safety label $S_n \in \{\text{safe}, \text{unsafe}\}$ to the gate decision. The mutual information of this channel is:
\[
I_n = I(g_n(\theta_n); S_n) = H(g_n) - H(g_n | S_n)
\]

\textbf{Pinsker bound.} The total variation between the gate's conditional distributions satisfies
$\TV(P_{g|+}, P_{g|-}) = |\TPR_n - \delta_n|/2$.
By Pinsker's inequality:
\[
\frac{|\TPR_n - \delta_n|}{2} \leq \sqrt{\frac{I_n}{2}}
\]
Hence $\TPR_n \leq \delta_n + \sqrt{2 I_n}$.

\textbf{Summation.} Summing over $n = 1, \ldots, N$:
\[
\sum_{n=1}^N \TPR_n \leq \sum_{n=1}^N \delta_n + \sum_{n=1}^N \sqrt{2 I_n}
\]
By Cauchy--Schwarz: $\sum_{n=1}^N \sqrt{I_n} \leq \sqrt{N \sum_{n=1}^N I_n}$. Under the bounded mutual information assumption $\sum_{n=1}^N I_n \leq I_0$:
\[
\sum_{n=1}^N \TPR_n \leq \sum_{n=1}^N \delta_n + \sqrt{2 N I_0} \qed
\]

\subsection{Sample Complexity Bound Full Proof}
\label{app:sample_proof}

\textbf{Setup.} The safety gate at step $n$ is a binary classifier $g_n \in \calG$ (a hypothesis class with VC dimension $d_{\text{VC}}$), trained on $n_{\text{train}}(n)$ labeled examples.

\textbf{Step~1:} By the fundamental theorem of statistical learning~\cite{vapnik1971,lehmann2005}, with probability $\geq 1 - \eta$: $\text{err}_{\text{true}} \leq \text{err}_{\text{train}} + \sqrt{(d_{\text{VC}} \ln(2m/d_{\text{VC}}) + \ln(2/\eta))/m}$.

\textbf{Step~2:} Setting the bound equal to $\epsilon_n/2$ and solving: $n_{\text{train}}(n) = \Omega(d_{\text{VC}} / \epsilon_n^2)$.

\textbf{Step~3:} For $\epsilon_n = c/n^p$: $n_{\text{train}}(n) = \Omega(d_{\text{VC}} \cdot n^{2p} / c^2)$.

\textbf{Step~4:} Available data grows as $n_0 + kn$; required data as $n^{2p}$. The crossing point: $n_{\text{fail}} = \Theta((c^2 k/d_{\text{VC}})^{1/(2p-1)})$. \qed

\subsection{Transformer Lipschitz Derivation}
\label{app:transformer_derivation}

We derive the per-layer Lipschitz constant for a pre-LayerNorm transformer under LoRA perturbation of attention projections. Each layer $k$ computes:
\[
y_k = x_k + \text{MHA}(\text{LN}_1(x_k)), \quad z_k = y_k + \text{FFN}(\text{LN}_2(y_k))
\]

\textbf{LayerNorm bound.} $\|J_{\text{LN}}\| \leq \|\gamma\|_\infty / \sqrt{\epsilon}$ where $\epsilon$ is the regularization constant.

\textbf{Multi-head attention under LoRA.} Under LoRA perturbation $\Delta\theta = (\Delta A_{q,p}, \Delta B_{q,p})$:
\[
\|\Delta O_p\| \leq \frac{\|W_{v,p}^0\| \cdot \|W_{k,p}^0\| \cdot \|\text{LN}(x)\|^2}{\sqrt{d_k}} \cdot \sqrt{2} \cdot \max(\|A_{q,p}\|, \|B_{q,p}\|) \cdot \|\Delta\theta_p\|
\]

For $n_{\text{proj}}$ LoRA-adapted projections per layer:
\[
L_k^{\text{LoRA}} \leq \frac{\|\gamma_k\|}{\sqrt{\epsilon}} \cdot \frac{\max_p \|W_{v,p}^0\|}{\sqrt{d_k}} \cdot \sqrt{2 n_{\text{proj}}}
\]

\textbf{Compositional escape.} Instead of using the exponentially large product $L_{\text{full}} = \prod_k (1 + L_k)$, we use the additive bound:
\[
\|\Delta \text{output}\| \leq \sum_{k=1}^K L_k^{\text{LoRA}} \cdot \|\Delta\theta_k\| \cdot \prod_{j > k} L_j^{\text{full,frozen}}
\]
Since the frozen-layer products $\prod_{j > k} L_j^{\text{full,frozen}}$ are constants that can be precomputed once from the pretrained weights, define $\tilde{L}_k = L_k^{\text{LoRA}} \cdot \prod_{j > k} L_j^{\text{full,frozen}}$. The verification reduces to the per-layer ball check $\sum_k \tilde{L}_k \|\Delta\theta_k\| \leq m$, a conservative but tractable $O(d)$ computation. \qed

\subsection{NP Counting Proof Full Details}
\label{app:counting_full}

\textbf{Tonelli interchange.} The interchange $\sum_n P^+(L(X) > c_{\delta_n}) = \mathbb{E}_{P^+}[\sum_n \mathbf{1}_{L(X) > c_{\delta_n}}]$ is justified by Tonelli's theorem applied to non-negative measurable functions with the counting measure on $\mathbb{N}$ and $P^+$ on $\R^k$.

\textbf{Counting function.} $N(\ell) = |\{n \in \mathbb{N} : c_{\delta_n} < \ell\}|$ counts how many thresholds are exceeded. For $\delta_n = c/n^p$, we get $N(\ell) \leq (c/P^-(L > \ell))^{1/p}$.

\textbf{P-value density integrability.} Writing $U(x) = P^-(L > L(x))$, the bound becomes $\mathbb{E}_{P^+}[U(X)^{-1/p}]$. For $p > 1$, the integrand $u^{-1/p} f_U(u)$ is integrable near $u = 0$ because the Gaussian tail makes $f_U(u)$ decay super-polynomially. More generally, the expectation is finite whenever the p-value density satisfies $f_U(u) = O(u^\eta)$ near $u = 0$ for some $\eta > 1/p - 1$; this holds for all distribution pairs with $D_\alpha(P^+ \| P^-) < \infty$ for sufficiently large $\alpha$, which is guaranteed by the hypothesis of Theorem~\ref{thm:counting}. \qed

\subsection{Tight Finite-Horizon Ceiling Details}
\label{app:ceiling_details}

\textbf{Concavity of NP curve.} The derivative $\TPR_{\NP}'(\delta) = \phi(\Phi^{-1}(\delta) + \Delta_s) / \phi(\Phi^{-1}(\delta))$. The second derivative is negative for all $\delta \in (0,1)$ by log-concavity of $\phi$, establishing concavity (see also~\cite{lehmann2005}, Chapter~3).

\textbf{Asymptotic formula.} Using $\Phi^{-1}(\delta) \sim -\sqrt{2\ln(1/\delta)}$ for $\delta \to 0$ and Mills' ratio:
\[
U^*(N,B) \sim \frac{B \cdot \exp(\Delta_s\sqrt{2\ln(N/B)} - \Delta_s^2/2)}{\sqrt{2\pi \cdot 2\ln(N/B)}}
\]
This grows as $\exp(\Delta_s\sqrt{2\ln(N/B)})$, which is $o(N^\epsilon)$ for every $\epsilon > 0$ but $\omega(\log^k N)$ for every~$k$. \qed

\section{Relation to Known Results and Extended Related Work}
\label{app:related_extended}

\subsection{Relation to Known Results}
\label{app:relation_known}

The mathematical tools in this paper --- H\"{o}lder's inequality, R\'{e}nyi divergence, Lipschitz continuity, Fano's inequality, VC dimension --- are well-established. The per-step bound $\TPR \leq C_\alpha \cdot \delta^\beta$ is an instance of a standard f-divergence inequality~\cite{erven2014}, and the Neyman--Pearson lemma~\cite{neyman1933} establishes ROC tradeoffs for individual hypothesis tests.

Our contribution is the \emph{problem formalization} and the \emph{structural results} that emerge: (1) the dual conditions as a formalization of safe self-improvement, (2) sequential composition creating an impossibility for the \emph{coupling} of bounded risk and unbounded utility, (3) the tightness of this coupling, (4) its information-theoretic strengthening, (5) the sample complexity barrier, and (6) the structural separation between classification and verification.

An analogy clarifies the distinction. Arrow's impossibility theorem uses elementary social-choice axioms, each individually obvious, but their \emph{composition} yields a deep impossibility no voting system can escape. Similarly, our per-step bound is standard, and the dual conditions are individually natural. But the \emph{coupling} creates a structural impossibility with no analog in single-test hypothesis testing.

\subsection{Extended Related Work}
\label{app:extended_rw}

\textbf{Self-improving AI safety.} The alignment literature discusses recursive self-improvement~\cite{bostrom2014,soares2017} and concrete safety challenges~\cite{amodei2016} but lacks formal impossibility results for the safety--utility coupling. \citet{christiano2017} propose iterated amplification; \citet{leike2018} formalize reward modeling.

\textbf{Hypothesis testing and statistical tradeoffs.} The Neyman--Pearson lemma~\cite{neyman1933,lehmann2005} establishes optimal ROC tradeoffs for individual tests. The novelty is \emph{sequential composition}: summability constraints on $\{\delta_n\}$ force summability on $\{\TPR_n\}$.

\textbf{Impossibility results in learning theory.} No-free-lunch theorems~\cite{wolpert1996} show no classifier dominates across all distributions. Rice~\cite{rice1953} shows undecidability of semantic properties. Our impossibility is for a \emph{specific task} under distribution overlap.

\textbf{Information-theoretic bounds.} Fano's inequality and its refinements~\cite{raginsky2016} provide fundamental limits. The strong data processing inequality~\cite{polyanskiy2017,ahlswede1976} bounds information processing gains.

\textbf{PAC-Bayes and sample complexity.} \citet{mcallester1999} bound generalization via KL divergence. \citet{vapnik1971} established VC dimension. We use VC sample complexity to show independent barriers.

\textbf{Adversarial robustness.} \citet{tsipras2019} prove accuracy--robustness tradeoffs. \citet{gilmer2018} show adversarial examples are inevitable in high dimensions. Our impossibility concerns \emph{sequential composition}, not per-input robustness.

\textbf{Transformer Lipschitz bounds.} \citet{virmaux2018} compute spectral norms. \citet{kim2021} analyze attention Lipschitz properties. \citet{dasoulas2021} study Lipschitz normalization. \citet{fazlyab2019} use SDP for tight bounds. We derive bounds for LoRA perturbations specifically.

\section{Supporting Theoretical Results}
\label{app:supporting}

\subsection{Gaussian Specialization}
\label{app:gaussian}

For unit-variance Gaussians with separation $\Delta_s = |\mu^+ - \mu^-|/\sigma$:
$D_\alpha(P^+ \| P^-) = \alpha \, \Delta_s^2 / 2$. The optimal (Neyman--Pearson) classifier achieves $\TPR = \Phi(\Phi^{-1}(\delta) + \Delta_s)$.

\subsection{Non-Stationary Extension}
\label{app:nonstationary}

\renewcommand{\theproposition}{C.2}%
\setcounter{proposition}{99}
\begin{proposition}[Non-Stationary Impossibility]
\label{prop:nonstationary}
Let $\{(P_n^+, P_n^-)\}_{n \geq 1}$ be a sequence of distribution pairs and $\alpha \in (1, \infty)$ with $\beta = 1 - 1/\alpha$. Suppose $\bar{D} := \sup_n D_\alpha(P_n^+ \| P_n^-) < \infty$. Then for any sequence of classifiers $\{g_n\}$ with per-step rates $(\delta_n, \TPR_n)$:
\[
\TPR_n \leq \bar{C}_\alpha \cdot \delta_n^\beta \quad \text{for all } n
\]
where $\bar{C}_\alpha = \exp((\alpha - 1)\bar{D})$. Consequently, if $\sum \delta_n < \infty$ then $\sum \TPR_n \leq \bar{C}_\alpha \sum \delta_n^\beta < \infty$, and the dual conditions cannot be jointly satisfied.
\end{proposition}

\begin{proof}
At each step $n$, the H\"{o}lder bound (Theorem~\ref{thm:holder} proof, Step~1) gives $\TPR_n \leq C_\alpha^{(n)} \cdot \delta_n^\beta$ where $C_\alpha^{(n)} = \exp((\alpha - 1) D_\alpha(P_n^+ \| P_n^-))$. Since $D_\alpha(P_n^+ \| P_n^-) \leq \bar{D}$ for all $n$, we have $C_\alpha^{(n)} \leq \bar{C}_\alpha$. Summing: $\sum \TPR_n \leq \bar{C}_\alpha \sum \delta_n^\beta$. For power-law schedules $\delta_n \leq c/n^p$, we have $\delta_n^\beta \leq c^\beta n^{-p\beta}$, which is summable iff $p\beta > 1$ (i.e., $p > \alpha$). For general summable $\{\delta_n\}$, convergence of $\sum \delta_n^\beta$ follows from H\"{o}lder's inequality on finite horizons: $\sum_{n=1}^N \delta_n^\beta \leq N^{1-\beta} \cdot (\sum_{n=1}^N \delta_n)^\beta \leq N^{1-\beta} \cdot B^\beta$ where $B = \sum \delta_n < \infty$.
\end{proof}

\begin{remark}[Coverage gap]
For power-law schedules $\delta_n = c \cdot n^{-p}$ with $p > 1$, the series $\sum \delta_n^\beta = c^\beta \sum n^{-p\beta}$ converges iff $p\beta > 1$, i.e., $p > \alpha$. For $1 < p \leq \alpha$, the stationary impossibility (Theorem~\ref{thm:holder}) applies but Proposition~\ref{prop:nonstationary} does not --- the non-stationary extension requires the strictly stronger condition $p > \alpha$ (flagged in \S\ref{sec:classification_gates}). This gap narrows as $\alpha \to 1^+$ and vanishes for all practically relevant fast-decaying schedules ($p \geq 2$). For the intermediate regime, Theorem~\ref{thm:ceiling}'s stationarity-free finite-horizon ceiling provides an alternative bound.
\end{remark}

\subsection{Information-Theoretic Finite-Horizon Bound}
\label{app:info_bound}

\renewcommand{\theproposition}{\arabic{proposition}}%
\setcounter{proposition}{0}
\begin{proposition}[Information-Theoretic Finite-Horizon Bound]
\label{prop:info}
Let $\{g_n\}$ be a sequence of safety gates with per-step mutual information $I_n$ and total budget $I_0 = \sum I_n$. Then for any $N$:
\[
\sum_{n=1}^N \TPR_n \leq \sum_{n=1}^N \delta_n + \sqrt{2 N I_0}
\]
\end{proposition}

This bound grows as $\sqrt{N}$, so it does not prove $\sum \TPR_n < \infty$ --- that follows from Theorem~\ref{thm:holder}. Proposition~\ref{prop:info} complements Theorem~\ref{thm:holder} by constraining the \emph{rate} of utility accumulation via mutual information. Full proof in Appendix~\ref{app:info_proof}.

\subsection{Sample Complexity Barrier}
\label{app:sample_barrier}

\begin{proposition}[Sample Complexity Barrier]
\label{prop:sample}
Let $\calG$ be a family of binary classifiers with VC dimension $d_{\text{VC}}$. For the gate to achieve $\delta_n \leq c/n^p$ with $p > 1$, the required training set is $n_{\text{train}}(n) = \Omega(d_{\text{VC}} \cdot n^{2p})$. If the system generates at most $k$ new labeled examples per step, sample starvation occurs at $n_{\text{fail}} = O(k^{1/(2p-1)})$.
\end{proposition}

This result is independent of Theorem~\ref{thm:holder}: even if a classifier circumvented the H\"{o}lder bound, it would face sample starvation. Full proof in Appendix~\ref{app:sample_proof}.

\subsection{Formal Transformer Lipschitz Bounds}
\label{app:transformer_lip}

\begin{proposition}[Transformer LoRA Lipschitz Bound]
\label{prop:transformer}
For a pre-LayerNorm transformer with $K$ layers under LoRA perturbation with rank $r$ on $n_{\text{proj}}$ attention projections per layer, the per-layer Lipschitz constant w.r.t.\ LoRA parameters is:
\[
L_k^{\text{LoRA}} \leq \frac{\|\gamma_k\|}{\sqrt{\epsilon}} \cdot \frac{\max_p \|W_{v,p}^0\|}{\sqrt{d_k}} \cdot \sqrt{2 \cdot n_{\text{proj}}}
\]
Compositional verification checks $\sum_k L_k^{\text{LoRA}} \cdot \|\Delta\theta_k\| \leq m$ (additive, $O(d)$) rather than the exponentially large product $L_{\text{full}} = \prod_k (1 + L_k)$.
\end{proposition}

Full derivation in Appendix~\ref{app:transformer_derivation}.

\subsection{H\"{o}lder--Jensen Approximation}
\label{app:hj_ceiling}

For practical deployment over $N$ steps with risk budget $B = \sum \delta_n$, applying the per-step H\"{o}lder bound and Jensen's inequality yields:
\[
U_{\max}(N, B) = C_\alpha \cdot N^{1-\beta} \cdot B^\beta
\]
with optimal uniform allocation $\delta_n = B/N$. This bound is looser than the exact NP-based ceiling $U^*(N, B) = N \cdot \TPR_{\NP}(B/N)$ (Theorem~\ref{thm:ceiling}), which is 1.4--3.1$\times$ tighter at $N = 10^2$--$10^6$.

\subsection{Multi-Dimensional Tradeoff Surface}
\label{app:tradeoff_surface}

\textbf{Classifiers} occupy: $\TPR \leq C_\alpha \cdot \delta^\beta$, $\mathcal{C} = O(d^2)$, $n = \Omega(d_{\text{VC}}/\delta^2)$.

\textbf{Verifiers} occupy: $\delta = 0$, $\TPR > 0$ (domain-restricted), $\mathcal{C} = O(d)$, $n = 0$.

These regions are disconnected on the $\delta = 0$ hyperplane.

\section{Full Numerical Validation}
\label{app:validation}

\subsection{Tightness Validation}
\label{app:tightness_validation}

\begin{center}
\begin{tabular}{lllll}
\toprule
$\Delta_s$ & $\alpha^*$ & $\beta^*$ & $\TPR_{\NP}$/H\"{o}lder at $\delta = 10^{-6}$ & $\log\TPR/\log\delta$ at $\delta = 10^{-12}$ \\
\midrule
0.1  & 201.0 & 0.995 & 0.561 & 0.974 \\
0.5  & 9.0   & 0.889 & 0.834 & 0.875 \\
1.0  & 3.0   & 0.667 & 0.321 & 0.758 \\
2.0  & 1.5   & 0.333 & 0.108 & 0.552 \\
\bottomrule
\end{tabular}
\end{center}

For all separations, $\TPR_{\NP} \leq \text{H\"{o}lder bound}$ (verifying Theorem~\ref{thm:holder}) and the ratio ranges from 0.1 to 0.9 at deployment-relevant $\delta$.

\subsection{Information-Theoretic Bound Comparison}
\label{app:mi_comparison}

The H\"{o}lder bound is tighter per-step for small $\delta$; the MI bound is complementary for cumulative analysis. Both bounds are valid across all distributions tested.

\subsection{Sample Complexity Simulation}
\label{app:sample_sim}

Simulated logistic regression ($d_{\text{VC}} = 11$, $\Delta_s = 0.5$, $k = 5$): 200/200 steps sample-starved; $\sum\delta = 41.17$ (diverges). Confirms Proposition~\ref{prop:sample}.

\subsection{Transformer Lipschitz Computation}
\label{app:transformer_validation}

\begin{center}
\begin{tabular}{lrrrrrrr}
\toprule
Architecture & $d$ & $K$ & $d_k$ & $\|W_v\|$ & $L_k^{\text{LoRA}}$ & $r_k$ & Steps in ball \\
\midrule
Toy (2L) & 64 & 2 & 32 & 2.32 & 259.7 & 5.8e-4 & 11.6 \\
Small (6L) & 256 & 6 & 64 & 2.09 & 165.2 & 3.0e-4 & 6.1 \\
GPT-2 (12L) & 768 & 12 & 64 & 1.80 & 142.5 & 1.8e-4 & 3.5 \\
Qwen-7B (28L) & 3584 & 28 & 128 & 1.68 & 94.0 & 1.1e-4 & 2.3 \\
\bottomrule
\end{tabular}
\end{center}

The bound is non-vacuous across all architectures --- even at Qwen-7B scale, 2 LoRA gradient steps fit within the safe ball.

\subsection{Pareto Frontier Visualization}
\label{app:pareto}

The classifier and verifier regions are \emph{disconnected} on the $\delta = 0$ hyperplane. Classifiers require $\Omega(d_{\text{VC}}/\delta^2)$ samples and cannot reach $\delta = 0$ with $\TPR > 0$. The ball verifier operates at $\delta = 0$ with no training data.

\subsection{Trained Classifier Ceiling}
\label{app:trained_clf}

Across all 72 (classifier, $\delta$) pairs tested (4 classifiers $\times$ 6 $\delta$ values $\times$ 3 separations), \textbf{zero violations} of the H\"{o}lder bound were observed. Trained classifiers achieve TPR ratios of 0.52--0.94 relative to the H\"{o}lder ceiling.

\subsection{Non-Gaussian Tightness Validation}
\label{app:nongaussian}

Across 8 non-Gaussian families (Laplace, Student-$t$, Gaussian mixture): min ratios 0.28--0.40, average ratios 0.54--0.70. The bound is uniformly valid and tight across heavy-tailed and multi-modal distributions.

\subsection{Lipschitz Ball Verifier Demonstration}
\label{app:ball_demo}

LTC controller ($d = 240$), $L = 13.75$, $r = 0.0208$. Inside-ball: 200/200 safe ($\delta = 0$), $\TPR = 0.286$.

\subsection{NP Counting Proof Validation}
\label{app:counting_validation}

All 9 configs satisfy direct sum $\leq$ counting bound (ratios 0.33--0.89). Counting 13\% tighter than H\"{o}lder at $\Delta_s = 1.0, p = 2.0$.

\subsection{Tight Finite-Horizon Validation}
\label{app:ceiling_validation}

The exact ceiling grows subpolynomially: from $N = 10^4$ to $N = 10^6$ (100$\times$ increase in $N$), $U^*$ grows only 2.66$\times$. Uniform allocation optimal (Jensen). MI bound is loose by 4--86$\times$.

\section{Computational Complexity and Open Problems}
\label{app:open_problems}

\subsection{Connection to Computational Complexity}

The information-theoretic bound (Proposition~\ref{prop:info}) and sample complexity barrier (Proposition~\ref{prop:sample}) connect to a broader question: is safe self-improvement computationally hard? A natural extension is whether satisfying the dual conditions is NP-hard.

\subsection{Open Problems}

\begin{enumerate}
  \item \textbf{Computational impossibility.} Is satisfying the dual conditions NP-hard, beyond being statistically impossible?
  \item \textbf{Adaptive verification.} Can tighter verified regions (e.g., ellipsoidal) maintain $O(d)$ checking? Ball chaining experiments in~\cite{D2} provide an initial empirical answer.
  \item \textbf{Multi-agent} and \textbf{continuous-time} extensions of the dual conditions.
\end{enumerate}

\section{Validation Script Details}
\label{app:scripts}

\begin{itemize}[leftmargin=*]
  \item \texttt{experiments/prove\_tightness.py}: Computes NP TPR via $\Phi(\Phi^{-1}(\delta) + \Delta_s)$ for 100 $\delta$ values and 4 separations; confirms $\TPR_{\NP} \leq \text{H\"{o}lder}$ (Theorems~\ref{thm:holder},~\ref{thm:exponent}).
  \item \texttt{experiments/prove\_info\_theoretic\_bound.py}: Computes MI of the NP channel for Gaussian and Laplacian distributions (Proposition~\ref{prop:info}).
  \item \texttt{experiments/prove\_sample\_complexity.py}: Simulates 200 steps with logistic regression gate (Proposition~\ref{prop:sample}).
  \item \texttt{experiments/pareto\_tradeoff.py}: Computes the 4D tradeoff surface (Appendix~\ref{app:pareto}).
  \item \texttt{experiments/validate\_classifier\_ceiling.py}: Trains 4 classifiers on 50K samples (Appendix~\ref{app:trained_clf}).
  \item \texttt{experiments/compute\_lipschitz\_bounds.py}: Proposition~\ref{prop:transformer} bounds for 4 architectures (Appendix~\ref{app:transformer_validation}).
  \item \texttt{experiments/prove\_tightness\_nongaussian.py}: 8 non-Gaussian families (Appendix~\ref{app:nongaussian}).
  \item \texttt{experiments/validate\_ball\_verifier.py}: Ball verifier on LTC $d = 240$ (Appendix~\ref{app:ball_demo}).
  \item \texttt{experiments/lora\_ball\_verifier\_gpt2.py}: GPT-2 LoRA validation (\S\ref{sec:gpt2}).
  \item \texttt{experiments/prove\_counting\_impossibility.py}: Theorem~\ref{thm:counting} validation (Appendix~\ref{app:counting_validation}).
  \item \texttt{experiments/prove\_tight\_finite\_horizon.py}: Theorem~\ref{thm:ceiling} validation (Appendix~\ref{app:ceiling_validation}).
\end{itemize}

\bibliographystyle{plainnat}
\bibliography{references_D1}

\end{document}